\documentclass[a4paper,reqno, final]{amsart}
\usepackage[centering,scale=0.7]{geometry}
\usepackage[foot]{amsaddr}
\usepackage{amsmath, amssymb, natbib, graphicx, url}
\usepackage[T1]{fontenc}
\usepackage[utf8]{inputenc}
\usepackage{mathptmx}
\usepackage{xpunctuate}
\usepackage{enumitem}
\usepackage{stmaryrd}
\usepackage{adjustbox}
\usepackage{graphicx}
\usepackage{natbib}
\usepackage{mathtools}
\usepackage{thm-restate}
\usepackage[ruled,vlined]{algorithm2e}
\usepackage[capitalize]{cleveref}
\usepackage{color}
\usepackage{etoolbox} 
\makeatletter
\patchcmd{\@algocf@start}{-1.5em}{0em}{}{} 
\makeatother
\SetAlCapHSkip{0pt} 
\SetAlgoHangIndent{0pt}
\SetInd{0.7em}{0.7em}
\newcommand{\BlockHeading}{}                    
\SetKwBlock{GenericBlock}{\BlockHeading}{}     
\newcommand{\VarBlock}[2]{%
  {\renewcommand{\BlockHeading}{\textnormal{#1}}%
   \GenericBlock{#2}%
  }%
}
\SetKwFor{RepeatIndef}{repeat}{}{end}
\SetKwFor{RepTimes}{repeat}{times}{end}

\theoremstyle{plain}

\newtheorem{proposition}{Proposition}

\theoremstyle{definition}
\newtheorem{definition}{Definition}

\newtheorem{example}{Example}
\theoremstyle{remark}
\newtheorem{remark}{Remark}

\newcommand{\eg}{e.g\xperiod}
\newcommand{\ie}{i.e\xperiod}
\newcommand{\IE}{I.e\xperiod}
\newcommand{\wrt}{w.r.t\xperiod}
\newcommand\etc{etc\xperiod}

\providecommand{\cond}{\,\vert\,}
\DeclarePairedDelimiterX{\infdivx}[2]{(}{)}{%
  #1\;\delimsize\|\;#2}
\newcommand{\kldiv}{D_{KL}\infdivx}
\DeclarePairedDelimiter{\norm}{\lVert}{\rVert}
\DeclarePairedDelimiter{\iverson}{\llbracket}{\rrbracket}
\DeclarePairedDelimiterX{\scalp}[2]{\langle}{\rangle}{#1,#2}

\begin{document}
\title[Deep Multivariate Models with Parametric Conditionals]{Deep Multivariate Models with Parametric Conditionals}
\author{Dmitrij Schlesinger}
\address[D.S.]{Dresden University of Technology}
\email{dmytro.shlezinger@tu-dresden.de}
\author{Boris Flach}
\email{flachbor@fel.cvut.cz}
\author{Alexander Shekhovtsov}
\address[B.F., A.S.]{Czech Technical University in Prague}
\email{shekhole@fel.cvut.cz}
\begin{abstract}
  We consider deep multivariate models for heterogeneous collections of random variables. In the context of computer vision, such collections may \eg consist of images, segmentations, image attributes, and latent variables. When developing such models, most existing works start from an application task and design the model components and their dependencies to meet the needs of the chosen task. This has the disadvantage of limiting the applicability of the resulting model for other downstream tasks. Here, instead, we propose to represent the joint probability distribution by means of conditional probability distributions for each group of variables conditioned on the rest.  Such models can then be used for practically any possible downstream task. Their learning can be approached as training a parametrised Markov chain kernel by maximising the data likelihood of its limiting distribution. This has the additional advantage of allowing a wide range of semi-supervised learning scenarios.
\end{abstract}

\maketitle

\section{Introduction}
When modelling joint probability distributions, say $p(x,y)$, in terms of deep networks, usually, the first question is how to factorise it, \ie to decompose it into parts which are then modelled by deep networks. The choice of this decomposition is  almost always dictated by the intended inference problem. 
For example, if $x$ is observable, $y$ is a class label, and the intended task is classification, the appropriate factorisation is $p(x,y) = p(x)p(y\cond x)$. If, on the other hand, the intended task is conditional generation, the reverse factorisation $p(x,y) = p(y)p(x\cond y)$ is suitable. The crucial point is that there is no single factorisation that enables solving both inference problems. If \eg $p(x\cond y)$ is modelled in terms of a deep network, the reverse conditional $p(y\cond x)$ is not accessible even in the case of a simple model for $p(y)$. 

The issue becomes even more difficult for more than two variables,
because the number of possible factorisations grows exponentially with the number of variables. Moreover, a particular factorisation admits only a few inference problems. For example, the factorisation $p(x,y,z) = p(x) p(y\cond x) p(z\cond x,y)$ admits predicting $y$ from $x$, but not $y$ from $x$ and $z$. Performing the latter requires the conditional $p(y\cond x,z)$, which is not directly modelled in the chosen factorisation. Its computation from the learned factors is usually infeasible if the latter are modelled by complex deep networks. A typical example is variational autoencoders (VAE) \citep{Kingma:ICLR14,Rezende:ICML14} and conditional VAEs \citep{Sohn:NIPS2015, Kingma-14}. They establish a (statistical) relation between images, classes and latent variables by choosing a factorisation suitable for the respective main goal (\eg semi-supervised model learning for conditional generation and inference). Tasks that are not directly accessible by the chosen model factorisation are then delegated to auxiliary encoders, which are used for lower bounding the training objective without being part of the model. This has the disadvantage that they are not consistent with the model and often require additional terms in the learning objective.  

Our main desiderata are therefore to model complex multivariate joint probability distributions such that they can be (i) used for various downstream inference tasks and (ii) trained with arbitrary levels of supervision. For this, we propose to represent joint probability distributions by means of conditional probability distributions for each group of variables conditioned on the rest. For the example above, we would represent $p(x,y,z)$ by means of three conditionals $p(x \cond y,z)$, $p(y \cond x,z)$ and $p(z \cond x,y)$. The joint distribution $p(x,y,z)$ is then defined as the limiting distribution of an MCMC process that alternately samples from the given conditionals. We propose an algorithm that learns the conditional distributions by maximising the data likelihood of the limiting distribution while encouraging them to be consistent with each other. Inference according to any conditional distribution, \eg $p(x\cond y)$, becomes accessible via the MCMC process with clamped inputs.

The rest of the paper is organised as follows. \cref{sec:related} reviews the related work. In \cref{sec:limiting} we consider learning limiting distributions defined by parametrised MCMC transition probabilities in a general setting, without partitioning variables into groups, and propose a recurrent lower bound approach. \cref{sec:multivar} details the approach for the case of multivariate joint probability distributions, shows a connection between consistency and detailed balance, and proposes a learning algorithm applicable with arbitrary levels of supervision. 
The experimental \cref{sec:experiments} validates the properties of the learning algorithm on artificial data and proposes proof-of-concept experiments on vision-related datasets.

\section{Related Work}\label{sec:related}

There have been several lines of work related to learning a joint distribution of heterogeneous observations and latent variables and addressing inference problems for arbitrary subsets of variables.

Most notably, multimodal VAEs \citep{Suzuki2016JointML,Wu-18,Shi-19,sutter2021generalized,sutter2024unity} have been proposed to constructively address cross-modality inference and semi-supervised learning by combining  elementary single-modality encoders via product of experts and other models. The solution that we propose is conceptually different. While we also optimize the ELBO on the likelihood of the (partially observed) training data, we leverage a factorisation-free representation of the joint probability distribution. Let us denote the modalities by $x_1,\ldots,x_n$ and the latent variables by $z$. In contrast to \citep{Suzuki2016JointML,Wu-18,Shi-19,sutter2021generalized,sutter2024unity}, we do not need to assume that the joint factorises as $p(z)p(x_1\cond z)...p(x_n\cond z)$. This assumption imposes a constraint that all cross-modality correlations must be captured by the latent variables. Neither do we need encoders $q$, modelled separately from the generative model $p$. This avoids the problem of per-modality encoders not learning through the combined joint encoder (MVAE~\citep{Wu-18}). The inference in our model, powered by MCMC, is more costly, but it can model complex posteriors even when the model components (conditional distributions for the variable groups) stay simple, \eg fully factorised models. In contrast, in multi-modal VAEs the required posterior $p(z \cond x_i)$ --- which necessarily needs to be complex for complex problems --- is approximated by a simple variational family, i.e. Gaussian $q(z \cond x_i)$.
Furthermore, we can solve many finer-grained (pixel-level) inference problems, as will be demonstrated. In \citep{Suzuki2016JointML,Wu-18,Shi-19,sutter2021generalized,sutter2024unity}, such sub-modality inference problems cannot be addressed. The proposed method thus can be more suitable for certain applications: requiring inference of more complex posteriors and handling more fine-grained partial observations, if the inference time is not critical. Our experiments aim to demonstrate these abilities, albeit on rather standard problems.

Actually, the idea to represent a joint probability distribution by means of its conditionals only is not entirely new, see \citep{DBLP:conf/uai/HeckermanCMRK00}. However, they did not go beyond the simplest pseudo-likelihood estimator for the fully supervised case, which is in turn a special case of our approach (see \cref{ex:PLLH} in \cref{app:examples}).

Next, we briefly discuss other related approaches which have certain commonalities with our method but do not align with our desiderata. 
%

JointGAN~\citep{Joint-GAN} trains multiple factorisations simultaneously (using a common critiq). These factorisations, though sharing some feature encoders, are separate, and their number grows combinatorially with the number of modalities and possible inference tasks. Furthermore, GANs have no variational inference mechanism to integrate out or condition on missing variables. The critiq assumes complete training modalities, which prohibits semi-supervised learning.

Related to the MCMC methods, kernels with inconsistent conditions have been used and analyzed in practice~\citep{Chen-15}. In our approach, the existence of a limiting distribution is guaranteed by the design of randomized sampling, the conditionals are parameterized by neural networks, and, while inconsistent initially, their consistency is enforced by the learning via the proposed theoretical connection between detailed balance and consistency and the recurrent ELBO optimization algorithms.

\citet{Liu-21} and \citet{Lamb-17} consider models related to the limiting distribution of a Markov chain cycling through conditionals. \citet{Liu-21} use a flow-based model for the encoder and propose a consistency-enforcing loss as well as ways to sample from model marginals {\em assuming that consistency holds}. GibbsNet by~\citep{Lamb-17} follows a GAN-based approach. These approaches are limited to the bivariate case.


\citet{Flach-24} formulates the learning as a two-player game between an encoder and decoder. In the simple case of unsupervised learning, this leads to the wake-sleep algorithm \citep{Hinton-95}. Our \cref{alg:multi} is a generalisation: it recovers these previous algorithms in the case of two variables when terminating the MCMC chain at length 1. The chain length allows for balancing consistency vs.~likelihood trade-off, previously possible only with more complex approaches \citep{Liu-21}.


\section{Learning Limiting Distributions} \label{sec:limiting}
  Let $p_\theta(x \cond x')$, $\theta\in\Theta$, $x,x'\in \mathcal{X}$, be a parametrised transition probability matrix defining a Markov chain with a possibly large but finite state space $\mathcal{X}$. We assume that it has a unique limiting distribution\footnote{\IE when some power of the transition probability matrix is strictly positive.} denoted by $p_\theta(x)$. Given a data distribution $q(x)$, we want to train the implicit model $p_\theta(x)$ by maximising the likelihood 
\begin{equation}
 L(\theta) = \sum_{x\in\mathcal{X}} q(x) \log p_\theta(x) \rightarrow \max_{\theta\in\Theta} .
\end{equation}
We assume that sampling from $p_\theta(x \cond x')$ and computing its derivative \wrt $\theta$ are both possible. However, this does not yet allow us to maximise the objective $L(\theta)$ directly. To approach the task, we propose to use a lower bound similar to variational inference.

Let $q(x \cond x')$ be another Markov chain model on $\mathcal{X}$ with a unique limiting distribution. Starting from the given data distribution, \ie $q_1(x) \coloneqq q(x)$, we define the following two sequences of distributions $q_k(x)$ and $Q_k(x)$, $k=2,3,\ldots$ recursively:
\begin{equation}
 q_k(x) = \sum_{x'\in\mathcal{X}} q(x \cond x') q_{k-1}(x'), \ \ \ 
 Q_n(x) = \frac{1}{n} \sum_{k=1}^n q_k(x).
\end{equation}
Both sequences converge to the unique limiting distribution of the model $q(x \cond x')$. Using the {\em stationarity condition} $p_\theta(x) = \sum_{x'}p_\theta(x')p_\theta(x\cond x')$ for the limiting distribution $p_\theta(x)$, we can expand and lower bound $L(\theta)$ as follows:
\begin{align}
  & L(\theta)  = 
  \sum_{x\in\mathcal{X}} q_1(x) \log \sum_{x'\in\mathcal{X}} 
  p_\theta (x \cond x' ) p_\theta(x') \geqslant L_B(\theta, q), \\
  & L_B(\theta, q) = 
   \sum_{x\in\mathcal{X}} q_2(x) \log p_\theta(x) + \nonumber  
  \sum_{x, x'\in\mathcal{X}} q(x' \cond x) q_1(x) 
  \Bigl[
  \log p_\theta(x \cond x') - \log q(x' \cond x) \Bigr].
 \end{align}
 Repeating this step for the first term of $L_B(\theta, q)$ recursively, we obtain an $n$'th-order lower bound
\begin{align} \label{eq:chain}
 & L_B(\theta, q, n) = \sum_{x\in\mathcal{X}} q_{n+1}(x) \log p_\theta(x) +\\
 & \ \ n\hspace{-.4em} \sum_{x, x'\in\mathcal{X}} q(x' \cond x) Q_n(x) 
 \Bigl[\log p_\theta(x \cond x') - \log q(x' \cond x) \Bigr] . \notag
\end{align}
To analyse this lower bound, we use the following reformulation.
\begin{restatable}{proposition}{PKL}\label{prop:KL}
The lower bound~\eqref{eq:chain} has an equivalent reformulation
  \begin{align} \label{eq:chain_kl}
    L_B(\theta, q, n) = L(\theta) - n \sum\limits_{x\in\mathcal{X}} Q_n(x) \kldiv*{q(x' \cond x)}{p^r_\theta(x' \cond x)},
   \end{align}
  where $p^r_\theta(x' \cond x) = p_\theta(x \cond x')p_\theta(x')/p_\theta(x)$ is the transition kernel of the reverse Markov chain.
\end{restatable}
This shows that $L_B(\theta, q, n)$ is tight iff $q(x'\cond x)$ coincides with the reverse model $p^r_\theta(x'\cond x)$. 

We would like to approach maximising \eqref{eq:chain} by a generalised EM algorithm, improving $L_B(\theta, q, n)$ in $\theta$ for a fixed $q$ in one iteration and updating $q$ to make the bound tighter in the other. Computing the gradient of the lower bound is, however, no simpler than computing the gradient of the likelihood $L(\theta)$ because the first term of \eqref{eq:chain} involves the limiting distribution $p_\theta (x)$. The next proposition shows that this term can be omitted for large $n$.
\begin{restatable}{proposition}{PZero}\label{P:zero}
Let $q(x' \cond x)$ be such that its unique limiting distribution coincides with $p_{\theta}(x)$ for a given $\theta$.
Then the gradient of the first term of~\eqref{eq:chain} in $\theta$ asymptotically vanishes for $n\rightarrow \infty$. 
In particular, it holds that 
$\norm*{\nabla_\theta \sum_{x\in\mathcal{X}} q_n(x) \log p_\theta(x)}_1 \leqslant C\alpha^n$ 
with some $\alpha\in(0,1)$ related to the mixing time of the kernel $q(x'\cond x)$ and a positive constant $C$. 
\end{restatable}
This motivates the following generic learning algorithm, ignoring the first term of~\eqref{eq:chain} and estimating the gradient of the second term of~\eqref{eq:chain} stochastically, \ie by sampling from $q$ and differentiating $p_\theta(x\cond x')$:
\begin{enumerate}
  \item For the current parameter $\theta^{(t)}$ chose $q^{(t)}(x' \cond x)$ with limiting distribution $p_{\theta^{(t)}}(x)$.
  \item Use $q^{(t)}$ to sample a sequence $(x_1,\ldots,x_{n+1})$ starting from  $x_1 \sim q(x)$ -- the data distribution.
  \item Estimate the gradient of $L_B(\theta, q^{(t)}, n)$ as
   \begin{align}\label{eq:stoch-est}
    g = \sum_{k=1}^n \nabla_\theta \log p_{\theta}(x_{k} \cond x_{k+1}) .
   \end{align}
  \item Update $\theta^{(t+1)} = \theta^{(t)} + \lambda g$, where $\lambda$ is the learning rate.
\end{enumerate}
The proofs of the two propositions and further facts used for this algorithm can be found in \cref{app:proof-P1}.

The algorithm becomes particularly sound when sampling from the reverse chain is feasible, \ie when it is possible to use $q^{(t)}(x' \cond x) =  p^r_{\theta^{(t)}}(x' \cond x)$. 
In particular, this is the case when $p_\theta(x' \cond x)$ satisfies detailed balance\footnote{The Markov chain satisfies detailed balance, if the reverse model coincides with the forward model.}.
If we use $q^{(t)}(x \cond x') = p^r_{\theta^{(t)}}(x\cond x')$ in the algorithm, the lower bound $L_B(\theta,q^{(t)},n)$ becomes tight at $\theta^{(t)}$ and its gradient coincides with the gradient of the likelihood $L(\theta)$ at $\theta^{(t)}$, for any $n$. This choice of $q$ also satisfies \cref{P:zero} because the limiting distribution of the reverse chain is the same as the limiting distribution of the forward chain.

However, the accessibility of the reverse model is a rather limiting assumption, not satisfied in particular when the kernel $p_\theta(x \cond x')$ is modelled with deep networks, which is of central interest in this work. Instead, we propose reusing the forward model $p_{\theta^{(t)}}(x' \cond x)$ for $q$. In this case, $L_B(\theta, q^{(t)}, n)$ is not tight but is still a valid lower bound to maximise. Moreover, \cref{P:zero} still applies for this choice.
We analyse the behaviour of the algorithm with this choice, assuming that the family of models $p_\theta(x \cond x')$, $\theta\in\Theta$ is rich enough to combine expressive power with detailed balance, \ie contains a large subset $\Theta_{DB}\subset\Theta$ of models satisfying detailed balance.

We conjecture that the algorithm implies a bias towards detailed balance. As we discussed, up to the omitted term with vanishing gradient for $n\rightarrow \infty$, the algorithm can be interpreted as stochastic gradient ascent for~\eqref{eq:chain_kl}. Thus, at each iteration, it balances between maximising the likelihood $L(\theta)$ and minimising the KL divergence from $q(x' \cond x) = p_{\theta}(x' \cond x)$ to $p_{\theta}^r(x' \cond x)$, enforcing their equality, \ie the detailed balance. The KL divergence has a weight of $n$ in~\eqref{eq:chain_kl}. Therefore, the larger $n$, the stronger the force towards the detailed balance. The KL term and its gradient vanish only if $\theta\in\Theta_{DB}$. If this happens, we are effectively able to sample from the reverse chain correctly, and the algorithm's behaviour can be expected to follow the preceding case, \ie when sampling from the reverse model is feasible.

To summarise, we expect the algorithm with $q^{(t)}(x' \cond x) = p_{\theta^{(t)}}(x' \cond x)$ and large $n$ to seek a maximum likelihood solution amongst models satisfying detailed balance. In the case when the set $\Theta_{DB}$ is not rich enough, we have to use a smaller $n$ to allow for a reasonable trade-off between detailed balance and likelihood. However, choosing a smaller $n$ might impair the accuracy of gradient approximation. It is then up to the experiments to verify whether this trade-off still leads to useful models.

\begin{remark}
  For the sake of simplicity, we have presented the derivation of the proposed approach for finite state spaces $\mathcal{X}$. In doing so we used only well-established formalisms, like e.g. variational lower bound, KL-divergences, Bayes formula, etc. Therefore, our approach can be generalised for continuous state spaces $\mathcal{X}$, provided that the following additional assumptions are fulfilled.
  
  \noindent(1) In order to apply gradient-based methods for the maximum likelihood estimation, we require that the PDFs of all considered model parts are differentiable almost everywhere \wrt to their arguments and parameters. This requirement applies to the kernels under consideration and their respective limiting distributions.
  
  \noindent(2) Furthermore, we have to assume that all considered Markov chains are ergodic, \ie irreducible, positive recurrent and aperiodic. In addition, for the proof of the convergence bound in \cref{P:zero}, we need to assume that the chain $q_n$ is geometrically ergodic.

\end{remark}

\begin{remark}\label{rem:diffusion}
It is noteworthy that choosing an auxiliary model $q(x\cond x')$ with strictly positive KL-divergence $\kldiv*{q(x' \cond x)}{p^r_\theta(x' \cond x)}$ has the consequence that the gap between the likelihood and the lower bound \eqref{eq:chain_kl} grows approximately linearly in $n$. Interestingly, this does not necessarily lead to a breakdown of the approach. A prominent example is denoising diffusion \citep{Ho:NeurIPS2020,Rombach:CVPR2022} (see \cref{ex:diffusion} in \cref{app:examples}).
\end{remark}

\section{Multivariate Models with Parametric Conditionals}\label{sec:multivar}

As mentioned in the introduction, our central goal is to develop and learn deep multivariate models for heterogeneous collections of random variables. Let $x = (x_1, x_2\ldots x_m)$ be a collection of such
heterogeneous random variables\footnote{Here, we use lower indices to denote components in contrast to the previous section, where they were denoting time steps.}. Let $p_{\theta_i}(x_i \cond x_{-i})$, for $i=1\ldots m$, be a conditional probability distribution for each variable $x_i$ conditioned on all other variables and parametrised by its own parameter $\theta_i$.
We consider an MCMC process that alternates between sampling components $x_i$. The Markov kernel of this MCMC process can be formally written as a convex combination
\begin{equation}\label{eq:gibbs_transition}
    p_\theta(x \cond x') =
    \sum_{i=1}^m \alpha_i \:
    p_{\theta_i}(x_i \cond x'_{-i})
    \iverson{x_{-i}\,{=}\,x'_{-i}} ,
\end{equation}
where the indicator $\iverson{x_{-i}\,{=}\,x'_{-i}}$ expresses that all variables except $x_i$ remain unchanged and $\alpha_i$ defines the relative frequency of resampling of each component.
This process is similar to Gibbs sampling and allows addressing multiple inference problems by clamping the subset of observed variables $x_o$ and iteratively resampling the rest. Supervised learning of $\{\theta_i\}$ is possible via the methods of \cref{sec:limiting} and semi-supervised learning becomes possible by the expectation-maximisation principle, first completing missing parts of a training example according to the current model.

In the above outline, there is the following subtlety: if $\{\theta_i\}$ are arbitrary, differently clamped processes may be inconsistent. To formalise this, assume that the process with kernel ~\eqref{eq:gibbs_transition} has a unique limiting distribution $p_\theta(x)$. The question is whether the MCMC process with clamped $x_o$ will have the same limiting distribution as $p_\theta(x_{-o} \cond x_o)$ implied by $p_\theta(x)$. This will be the case when model conditionals are consistent, as we now explain.

\begin{definition}
    A collection of conditional distributions $p_{\theta_i}(x_i \cond x_{-i})$, $i=1,\ldots,m$ is {\em consistent} if there exists a joint distribution $p(x)$ that has them as conditionals.
\end{definition}
One example of a consistent system of conditionals is the EF-Harmonium~\citep{Welling:NIPS2004}, a generalisation of RBM (which we will consider in \cref{ex:EF} in \cref{app:examples}).
\citet{Arnold-2001,Shekhovtsov:ICLR2022} discuss such consistency for the case of two variates.
We note that if $p_{\theta_i}(x_i \cond x_{-i})$ are consistent, then an MCMC process with clamped $x_o$ becomes a Gibbs sampling scheme, guaranteed to sample from $p_\theta(x_{-o} \cond x_o)$. We further characterise this consistency in terms of detailed balance of the Markov chain.

\begin{restatable}{theorem}{Tconsistency}\label{T1}
    The Markov chain defined by the kernel \eqref{eq:gibbs_transition} satisfies detailed balance if and only if the collection of conditional distributions $p_{\theta_i}(x_i \cond x_{-i})$, $i=1,\ldots,m$ is consistent. In this case, all individually modelled conditionals  $p_{\theta_i}(x_i \cond x_{-i})$ coincide with the corresponding conditionals of the limiting distribution, $p_\theta(x_i \cond x_{-i})$.
\end{restatable}
The proof is given in \cref{app:proof-T1}. 
We see that the detailed balance of the Markov chain, which was desirable for learning in \cref{sec:limiting}, is also necessary for rigorous inference with the proposed method. As we have argued, the generic algorithm \eqref{eq:stoch-est} with the choice $q(x\cond x') = p_\theta(x\cond x')$ enforces detailed balance. The theorem implies that it also enforces consistency when applied to multivariate models; therefore, the choice $q(x\cond x') = p_\theta(x\cond x')$ is of particular interest. Next, we discuss semi-supervised learning cases.

\begin{algorithm}[t]
    \caption{Learning Joint Probability Distributions}
    \label{alg:multi}
    \DontPrintSemicolon
    \KwIn{training example $(x_o,x_{-o})$, where $x_o$ is the group of observed and $x_{-o}$ is the group of unobserved variables; conditional probability distributions $p_{\theta_i}(x_i \cond x_{-i})$ for each variable.}
    \BlankLine
    \VarBlock{{\bf I.} Complete the example by sampling over unobserved variables:}{
    Initialise $\hat x = (x_o,\hat x_{-o})$ with random $\hat x_{-o}$\;

    \RepTimes{$n_{\rm inference}$}{
        Randomly choose unobserved $i\in -o$\;

        Draw $\hat x_i \sim p_{\theta_i}(x_i \cond \hat{x}_{-i})$\;
    }
    }
    \BlankLine
    \VarBlock{ \textbf{II.} Learn the model on the completed example $\hat x$:}{
    \RepTimes{$n$}{
        Randomly choose a component $i \in \{1,\dots,m\}$\;

        Compute $g \gets \nabla_{\theta_i} \log p_{\theta_i}(\hat{x}_i \cond \hat{x}_{-i})$\;

        Update $\theta_i \gets \theta_i + \lambda g$\;

        Draw $\hat{x}_i \sim p_{\theta_i}(x_i \cond \hat{x}_{-i})$\;
    }
    }
\end{algorithm}

\subsection*{Practical Learning Algorithm.}
Without loss of generality, training examples can be assumed to have an observed part $x_o$ and an unobserved part $x_{-o}$.
This partitioning can be different for different examples, but it is considered to be independent of the data itself. The likelihood of the observed part is given by
\begin{align}\label{eq:multi1}
    \log p_\theta(x_o) = \log \sum_{x_{-o}} p_\theta(x_o, x_{-o}) .
\end{align}
Following the expectation-maximisation scheme, we use a stochastic approximation of its gradient
\begin{subequations}\label{eq:multi3}
    \begin{align}
    & \hat x_{-o} \sim p_{\theta^{(t)}}(x_{-o}|x_o), \\
    & \nabla_\theta \log p_\theta(x_o) \approx
    \nabla_\theta \log p_\theta(x_o, \hat x_{-o}),
    \end{align}
\end{subequations}
\ie, we first complete the example by sampling the missing values from ~$p_{\theta^{(t)}}(x_{-o} \cond x_o)$, according to the current model $\theta^{(t)}$. 
This is achieved by running the sampler \eqref{eq:gibbs_transition} with clamped values $x_o$ and is implemented in step I of \cref{alg:multi}. The required number of repeats, $n_{\rm inference}$, depends on the mixing time of the sampler \eqref{eq:gibbs_transition}. The same holds for the inference when we observe $x_o$ and want to predict $x_{-o}$.
Completed examples proceed to step II, implementing the method of \cref{sec:limiting} for $q(x\cond x') = p_\theta(x\cond x')$.

Instead of generating a chain of length $n$ for each individual training example, in step II we keep a batch of examples and in each iteration update the model parameters and resample all examples in the batch according to the transition kernel. In the actual implementation, we also replace a randomly chosen fraction of examples with fresh ones from the training data (suitably completed with step I). This way, the loop in step II can run continuously, with the chain for each example having some probability to terminate when it is replaced with a fresh one. We observed this to improve stability. As a consequence, in practice, we do not have an explicit chain length hyperparameter $n$ but an {\em expected chain length}, defined by the batch size and the replacement ratio. 

\begin{remark} \label{rem:inconsistent}
 The proposed semi-supervised learning approach achieves consistency by enforcing detailed balance, which follows from \cref{T1}. However it still has an apparent weakness in the completion step \eqref{eq:multi3} especially at the beginning of the learning if the initial model parts $p_{\theta_i}(x_i \cond x_{-i})$ are inconsistent. This is because the completion is achieved by using them in the clamped chain, which is not leading to the true conditional $p_\theta(x_{-o} \cond x_o)$ of the unclamped limiting distribution. We nevertheless conjecture that the proposed expectation-maximisation scheme converges to consistent models for sufficiently long chains. This is due to works as \citep{Allassonniere:CSDA2021}, which consider approximated stochastic EM algorithms. The sampling distribution for latent variables used in them deviates from the true posterior of the (current) model. The authors prove that the algorithm converges to (local) maxima of the likelihood if these sampling distributions converge to the true model posteriors during the learning.  In our case, enforcing detailed balance implies enforcing consistency and therefore, the model conditionals should converge to a consistent model (see \cref{subsec:consistency} for an experimental validation).
 \end{remark}

\begin{remark}
The (average) chain length $n$ has multiple effects: (1) Training time.
(2) A larger $n$ is needed for the approximation used in \cref{P:zero}, especially for models with a large mixing time.
(3) Larger $n$ are also required to achieve consistency (via enforcing detailed balance).
Currently, we do not have control over the mixing time and focus on the training methodology and proof-of-concept experiments, demonstrating that useful trade-offs do exist.
\end{remark}

\section{Experiments} \label{sec:experiments}


\subsection{Validation: Chain Length} \label{subsec:artificial}
This experiment investigates the behaviour of the algorithm proposed in \cref{sec:limiting} in the non-tight case with $q(x'\cond x) = p_\theta (x' \cond x)$ and aims to confirm its bias towards detailed balance.
We learn a simple log-linear transition probability distribution to match a given target distribution. 
Let $z \in \{-1,+1\}^n$ be a vector of binary variables (we use $n{=}64$ in this experiment), and $\pi(z) \coloneqq q_1(z)$ be the target probability distribution. The transition probability distribution is defined as
\begin{equation}\label{eq:art_mat}
    p_\theta(z_t \cond z_{t-1})\propto\exp \scalp{W z_{t-1} + b}{z_t} ,
\end{equation}
where the matrix $W\in \mathbb{R}^{n\times n}$ and the bias $b\in\mathbb{R}^n$ are unknown parameters, summarised by $\theta$. We use the marginal distribution of the latent variables $p(z)$ from our MNIST experiments as the target probability distribution $\pi(z)$ (see \cref{subsec:mnist} and \cref{app:artificial} for details).

\begin{figure}[ht]
    \centering
    \includegraphics[width=0.8\linewidth]{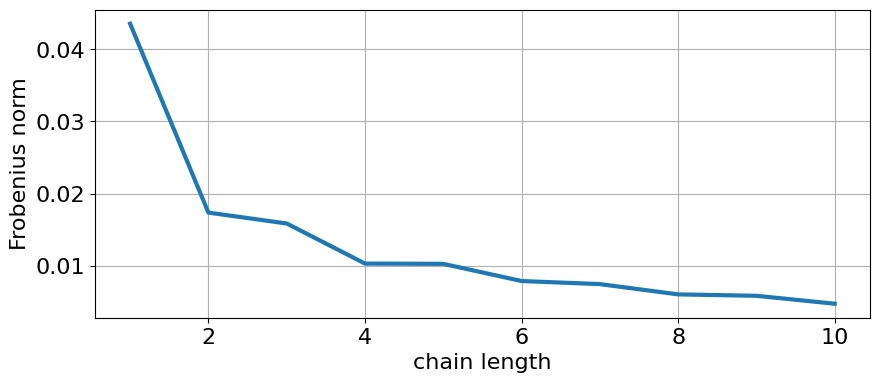}
    \caption{\label{fig:art3} Dependency of the matrix asymmetry on the chain length $n$; see \cref{subsec:artificial}.}
\end{figure}

We notice that the model satisfies detailed balance $p_\theta = p_\theta^r$ iff $W$ is symmetric. The asymmetry of $W$, measured with Frobenius norm $\norm{W - W^T}_F$, can therefore serve as a natural measure for detailed balance.
We perform a series of experiments varying the chain length $n$. \cref{fig:art3} shows the dependency of the resulting matrix's asymmetry on the chain length. As expected, the learned matrix $W$ becomes symmetric for longer chains. This confirms our conjecture that the algorithm implicitly enforces detailed balance, which tightens the lower bound, making it closer to the likelihood. Additional details and further experiments accomplishing these results can be found in \cref{app:artificial}.
\subsection{Validation: Enforcing Consistency} \label{subsec:consistency}
This experiment investigates the approach in the multi-modal case with particular focus on consistency. We learn a simple Gaussian mixture model (GMM) $p(x) = \sum_k p(k)\cdot \mathcal N(x; \mu_k, \sigma_k)$, $k=0,1$, $x\in\mathbb{R}$. In particular, we consider a decoder family $p(x\cond k) = \mathcal N(x; \mu_k, \sigma_k)$, where $\mu_k$ and $\sigma_k$ are learnable parameters. The corresponding encoder family $p(k\cond x)$ would be a set of quadratic logistic regressions. However, we restrict it, assuming only linear logistic regressions $p(k\cond x)\propto\exp\bigl[k\cdot(ax + b)\bigr]$ with the learnable parameters $a$ and $b$. Note that consistent encoder/decoder pairs satisfy
\begin{equation}\label{eq:gmm_consset}
    \sigma_0 = \sigma_1 = \sigma, \ \ \ a = \pm \frac{\mu_1-\mu_0}{\sigma^2} .
\end{equation}
The model is learned on a target distribution $p^\ast(x) = \sum_k p^\ast(k) p^\ast(x\cond k)$, which is a GMM of two Gaussians with different standard deviations. This means that the considered decoder family contains the ground truth decoder. Moreover, the model contains a decoder/encoder pair $p(x\cond k)=p^\ast(x\cond k)$ and $p(k\cond x)$ whose limiting marginal probability distribution $p(x)$ matches the target distribution $p^\ast(x)$. To achieve this, we can simply choose the encoder so that the marginal limiting distribution $p(k) = \sum_x p(x)\cdot p(k\cond x)$ coincides with the ground truth prior $p^\ast(k)$.
However, this ``optimal'' pair will be inconsistent because of the linear logistic regression.

\begin{figure}[ht]
    \centering
    \includegraphics[width=\linewidth]{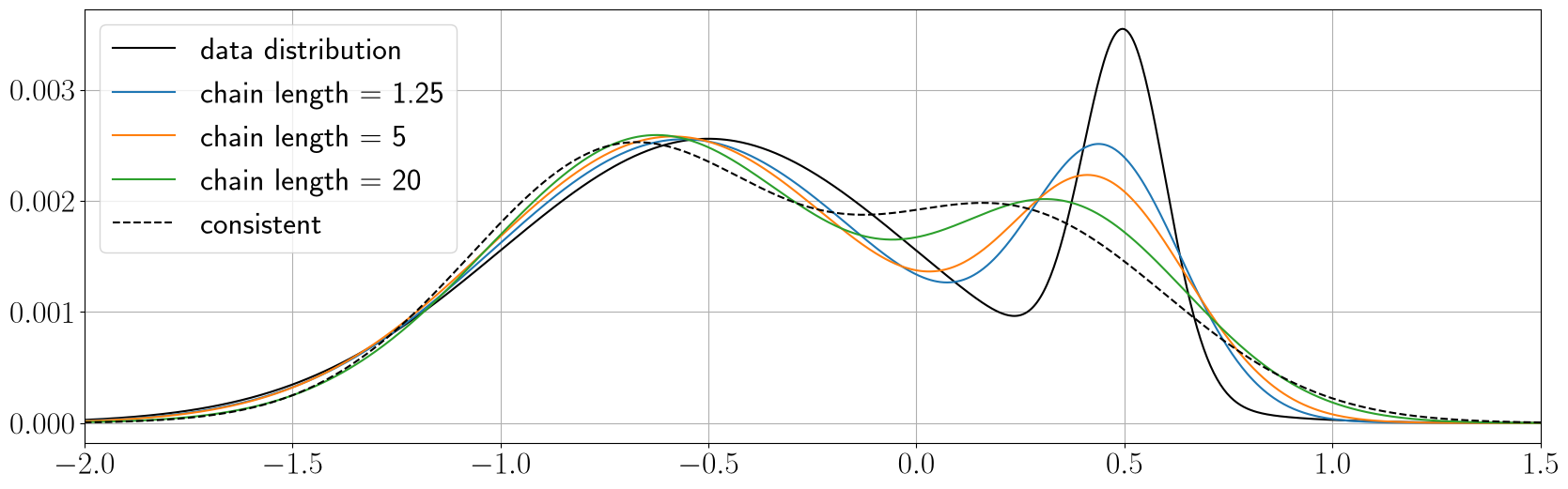}\\

    \vspace{1ex}
    \includegraphics[width=\linewidth]{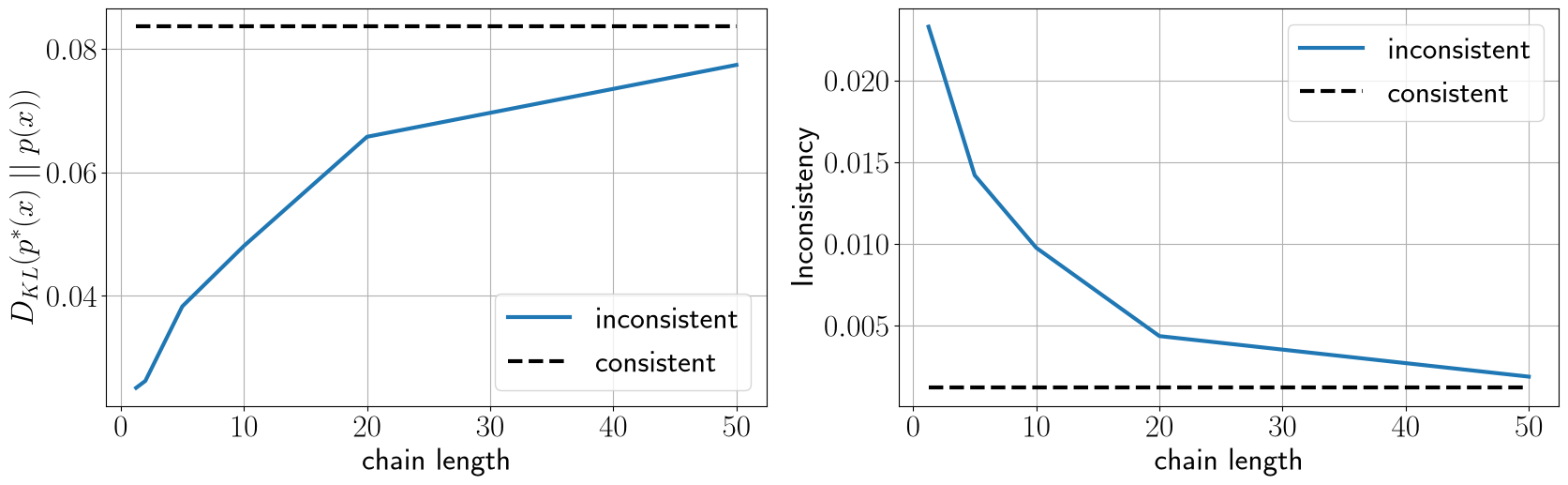}
    \caption{\label{fig:gmm} GMM results. Top: the learned densities (coloured lines), as well as the ground truth density (shown in black) and the density obtained for the best consistent model (shown as dashed black line). Bottom, left: KL-divergence $\kldiv{p^\ast(x)}{p(x)}$ between the learned and the ground truth densities in dependency on the chain length. Bottom, right: the dependency of the proposed inconsistency measure on the chain length.}
\end{figure}

We perform a series of experiments with different average chain lengths. In addition, we learn a model which is consistent by design, enforcing the requirements \eqref{eq:gmm_consset} after each gradient update. The resulting limiting densities $p(x)$ are presented in \cref{fig:gmm} (top). As expected, longer chains enforce consistency by pushing the learned decoder towards a model with equal standard deviations. At the same time, the quality of the learned marginal limiting distribution $p(x)$ deteriorates, which is visible in \cref{fig:gmm} (bottom, left). We measure it by the KL-divergence $\kldiv{p^\ast(x)}{p(x)}$.

In order to measure inconsistency quantitatively, we propose to proceed as follows. We notice that the limiting probability distribution of the sampler \eqref{eq:gibbs_transition} (in case of two groups of variables) is a mixture
\begin{equation}\label{eq:gmm_cons1}
    p(k, x) = \alpha \cdot p(k) p(x\cond k) + (1-\alpha)\cdot p(x) p(k\cond x) ,
\end{equation}
where $p(k)$ and $p(x)$ are the limiting marginal probability distributions. The latter also satisfy the equations 
\begin{equation}\label{eq:gmm_marg}
    p(x) = \sum_k p(k) p(x\cond k) \ \ \text{and} \ \ p(k) = \sum_x p(x) p(k\cond x) .
\end{equation}
This applies regardless of whether the decoder/encoder pair is consistent. In the consistent case, the two components in \eqref{eq:gmm_cons1} coincide. Therefore, we propose to use the symmetrized KL-divergence between two components in \eqref{eq:gmm_cons1} 
\begin{equation}\label{eq:gmm_cons2}
    \kldiv{p(k) p(x\cond k)}{p(x) p(k\cond x)} + \kldiv{p(x) p(k\cond x)}{p(k) p(x\cond k)}
\end{equation}
as an inconsistency measure. Given the simple models considered here, we can reliably compute the limiting marginal $p(k)$ by MCMC, and then also $p(x)$ as given in \eqref{eq:gmm_marg}. The dependency of this consistency measure on the chain length is shown in \cref{fig:gmm} (bottom, right).

Notice that the experiment involves unsupervised learning because the ground truth classes $k$ are not presented in the training data. Consequently, a given training example $x$ must first be completed by sampling the corresponding $k\sim p(k\cond x)$ (see \eqref{eq:multi3}). We initialise the parameters of the encoder and decoder by standard PyTorch random weight initialization, which means that at the beginning of the learning they are inconsistent. In particular, for the encoder/decoder families considered here, the initial inconsistencies \eqref{eq:gmm_cons2} were $0.23$ on average, which is far higher than the final inconsistencies shown in \cref{fig:gmm}. 
Consequently, the completion step is approximate and can in principle harm the learning. However, we see that the proposed approach nevertheless succeeds in enforcing consistency via enforcing detailed balance (see also the discussion in \cref{rem:inconsistent}).

\subsection{MNIST / Fashion MNIST} \label{subsec:mnist}
This experiment illustrates our approach on simple labelled image data.
The task is to learn a distribution $p(x,c)$, where $x$ are images and $c$ are classes. In particular, for MNIST, we consider binarised images $x\in\mathcal{B}^{28\times 28}$ (a common simplification) and classes $c\in\mathcal{C}=\{0,1,\ldots,9\}$ corresponding to digits. We introduce additional latent variables $z\in\mathcal{B}^m$ (we used $m{=}64$), which should ``encode'' different appearances for each symbol. Hence, the full joint model is $p(x,c,z)$. It is specified by means of the following three conditional probability distributions: 
\begin{eqnarray}\label{eq:mnist}
    & & p(x\cond c,z)\propto \exp\scalp{f_x(c_{oh},z)}{x} ,\nonumber \\
    & & p(c\cond x,z)\propto \exp\scalp{f_c(x) + z^T W_c}{c_{oh}}, \nonumber\\
    & & p(z\cond x,c)\propto \exp\scalp{f_z(x) + c_{oh}^T W_z}{z} ,
\end{eqnarray}
where $c_{oh}$ is one-hot encoded $c$. The functions $f_c\colon\mathcal{X}\rightarrow \mathbb R^{10}$ and $f_z\colon\mathcal{X}\rightarrow \mathbb R^m$ are convolutional feedforward networks with an encoder-like architecture, \ie with decreasing spatial resolution. The decoder $f_x\colon\mathcal{B}^{10+m}\rightarrow\mathbb R^{28\times 28}$ ($c_{oh}$ and $z$ are concatenated) consists of a few ``transposed convolutional'' layers, $W_c$ and $W_z$ are $m{\times}10$ and $10{\times}m$ matrices, respectively.

The model for Fashion MNIST is exactly the same except that we use grey-valued images $x\in\mathbb{R}^{28\times 28}$ and model $p(x\cond c,z)$ as a Gaussian $p(x\cond c,z)=\mathcal N(x;\mu(c,z),\sigma)$, where the means $\mu(c,z)$ for each pixel are provided by a feedforward network with the same architecture as $f_x$ in \eqref{eq:mnist} and $\sigma\in\mathbb R$ is a global learnable parameter. In addition, we increased the dimension of the latent space from $m=64$ to $m=256$.

The models were learned by \cref{alg:multi}. The completion step~I in these models is trivial -- for each training sample, we only need to sample once from $p(z\cond x,c)$. The optimisation step~II of the algorithm was performed with the average chain length $n=4$.

\begin{figure}[t]
    \centering
        \includegraphics[width=0.49\linewidth]{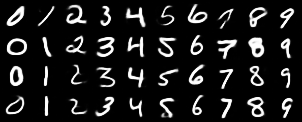}\ 
        \includegraphics[width=0.49\linewidth]{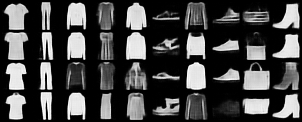}\\
    \caption{\label{fig:fmnist} Conditional generation for MNIST / Fashion MNIST. Columns correspond to classes.}
\end{figure}

The results of conditional generation are presented in \cref{fig:fmnist}. Given a class $c$, we sample images $x\sim p(x\cond c)$, which can be done by alternately sampling from $p(x\cond c,z)$ and $p(z\cond x,c)$, starting from a random $z$. For better visualisation, we show probabilities $p(x\cond c,z)$ for MNIST and means $\mu(c,z)$ for Fashion MNIST for the last sampled $z$. In order to rate the image quality quantitatively, we use the Fréchet Inception Distance (we used the code from \citep{Seitzer2020FID}). The obtained values are $5.87$ for MNIST and $30.49$ for Fashion MNIST.

Next, we consider the classification according to the maximum marginal decision $\arg\max_c\sum_z p(c,z\cond x)$, where marginalisation over $z$ is implemented by alternating sampling from $p(c\cond x,z)$ and from $p(z\cond x,c)$, starting from a random $z$. We achieve  $99.04$\% classification accuracy on MNIST and $88.61$\% on Fashion MNIST.


The main aim of the experiment was to illustrate the ``universality'' of our approach, \ie using one and the same model for different inference problems. We have not tuned the models towards a specific inference problem. Nevertheless, we find that the same simple model architecture works quite well for different datasets.\footnote{Further details can be found in the source code, which is available at https://github.com/dschles70/paracond.}

\subsection{CelebA} \label{subsec:celeba}
The next group of experiments illustrates our approach on more challenging data. We use the CelebA-HQ dataset, containing images of human faces, along with attributes and segmentations. In particular, we demonstrate the ability of our approach to cope with different downstream tasks.

We introduce the following random variables. Let $x\in\mathbb R^{3*64*64}$ be images (we downscaled images and segmentations to $64{\times}64$ resolution for simplicity), $s\in\mathcal{B}^{18*64*64}$ be segmentations (a binary mask for each of the 18 segments) and $c\in\mathcal{B}^{40}$ be image attributes, like \eg ``male'', ``blond'', ``wearing glasses'', \etc. In addition, we introduce two groups of latent variables: $y\in\mathcal{B}^{256}$ are binary and $z\in\mathbb{R}^{16}$ are continuous.

\begin{figure}[t]
    \centering
    \includegraphics[width=0.9\linewidth]{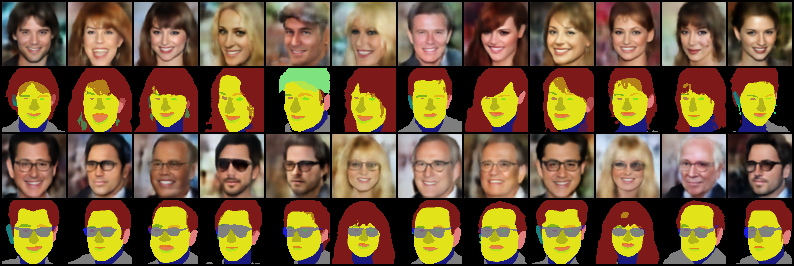}\\[5pt]
    \includegraphics[width=0.9\linewidth]{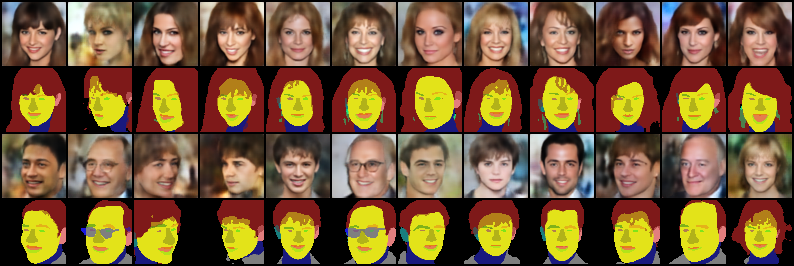}\\
    \caption{\label{fig:celeba_cond} Generation for CelebA, conditioned on the presence of glasses attribute (top) and male attribute (bottom).
    Each display shows pairs of image and segmentation generated with the attribute switched off (first two rows) and switched on (last two rows).
    }
\end{figure}
\begin{figure}[t]
\centering
\includegraphics[width=0.9\linewidth]{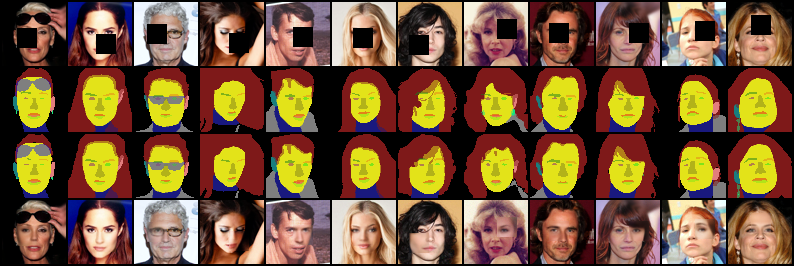}\\
\caption{\label{fig:celeba_segming} Segmentation from incomplete images / in-painting. First row: original images with hidden parts shown as black squares, second row: ground truth segmentations, third row: predicted segmentations,  fourth row: in-painting.}
\end{figure}
\begin{figure}[t]
\centering
\includegraphics[width=0.9\linewidth]{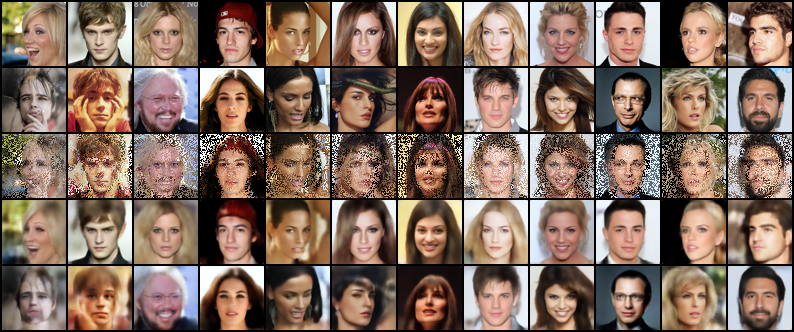}\\
\caption{\label{fig:celeba_demix} Image demixing. First two rows: original image pairs, third row: mixed image, last two rows: demixing results.}
\end{figure}

The joint model is represented by five conditional probability distributions, \ie for each variable conditioned on the rest. Each conditional is implemented as a feed-forward network of moderate complexity. Thereby, networks for binary variables output logits, whereas networks for continuous variables $x$ and $z$ output means and variances for each element, for example $p(x\cond c, y, z, s) = \mathcal{N}(x; \mu(c, y, z, s), \sigma(c, y, z, s))$. 

The model was learned by the \cref{alg:multi} for about 470k gradient updates (corresponding to 1000 epochs in our implementation). We used $10$ sampling iterations in the completion step I (we need to complete triples $(c,x,s)$ by sampling the latent variables $y$ and $z$) and the expected chain length $n=16$ in step II. Further details are given in \cref{app:nets}. In \cref{fig:celeba_cond} we demonstrate generation conditioned on attributes. The achieved FID score for the generated images is 46.45. This is considerably better than the best FID score ($\geqslant 60$) reported for CelebA in \citep{sutter2024unity} for multimodal VAEs.

The next experiment demonstrates the ability of the learned model to complete missing information for any subset of its variables, including any subset of their individual coordinates.
This problem cannot be approached by any standard factorisation, because the subset of unobserved variables is not known at training time. We illustrate this ability on the example of segmentation from incomplete images. Let $x = (x_o, x_h)$, where the part $x_o$ is observed and $x_h$ is hidden. In order to segment such an image by the maximum marginal decision strategy, we need to compute the marginal probabilities $p(s_{ik} \cond x_o)$ for each pixel $i$ and segmentation label $k$. They can be estimated by sampling, which alternately draws all unobserved random variables, including $x_h$, while clamping $x_o$. We accumulate segmentation label frequencies for each pixel during the sampling and finally set the label $k$ if its marginal probability is greater than $0.5$. At the same time, we complete the missing image content (\ie in-painting). For this, we employ the minimum mean square error decision, \ie we average all sampled image values observed during sampling. The results are illustrated in \cref{fig:celeba_segming}.

The final experiment shows how to extend the model to perform even more complex inference problems. We consider the following image ``demixing'' task. A mixed image $x$ is obtained from two images $x'$ and $x''$ by randomly choosing the RGB values in each pixel of $x$ from either the first or the second image. The task is to reconstruct the original images given the mixed one.

In order to define an appropriate statistical model for this task, we consider two instances of our trained model for the random variables $(c',x',y',z',s')$ and $(c'',x'',y'',z'',s'')$ for the first and the second image, respectively. In addition, we define a binary mask $m\in\mathcal{B}^{64*64}$ and assume a uniform prior $p(m)$. The mixed image is obtained as $x = x'*(1-m) + x''*m$. Even though the last expression is deterministic, it is easy to see how to sample all introduced components conditioned on $x$. For segmentations, attributes and latent variables, we can directly use the learned conditionals. In order to sample \eg from $p(x'\cond c',y',z',s', x, m)$, we proceed as follows. We first sample from the learned conditional $p(x'\cond c',y',z',s')$, and then replace $x'$ by $x$ in pixels with $m=0$ (similarly for $x''$). Finally, in order to sample from $p(m\cond \text{everything})$, we compute the likelihood ratio $\ln p(x\cond c'',y'',z'',s'') - \ln p(x\cond c',y',z',s')$, which is the logit for $p(m\cond \text{everything})$.
We notice that this demixing task is highly ambiguous. Therefore, to regularise and simplify it, we assume known segmentations $s'$ and $s''$. \cref{fig:celeba_demix} shows the results.

\section{Conclusion}
We consider deep multivariate models for heterogeneous collections of random variables and propose to model them in terms of conditional distributions for each variable conditioned on the rest. Learning such collections of conditional distributions can be approached as training a parametrised Markov chain kernel by maximising the data likelihood of its limiting distribution. The main advantage of the method is its "universality", \ie models with the ability to solve different tasks which can be learned with arbitrary supervision levels. 


Our experiments show that such models can be learned for simple and complex visual data. Particularly, we focus on the versatility of the learned models \wrt different inference tasks. Moreover, we show that once trained, they can be easily extended for more complex downstream tasks. 






\bibliographystyle{plainnat}
\bibliography{strings,paracond}

\clearpage

\appendix
\crefalias{section}{appendix}
\crefalias{subsection}{appendix}

\section{Proofs and Examples}\label{app:proof}

\subsection{Proofs for the Generic Algorithm}\label{app:proof-P1}
The following result is a direct analogue of ELBO decomposition, applied to our recurrent lower bound.
\PKL*
\begin{proof}
 We use the equality
  \begin{equation*}
    \log p_\theta(x \cond x' ) = \log p^r_\theta(x' \cond x) - 
    \log p_\theta(x') + \log p_\theta(x)
  \end{equation*}
 to substitute $\log p_\theta(x \cond x')$ in 
 \begin{equation}
  L_B(\theta, q, n) = \sum_{x\in\mathcal{X}} q_{n+1}(x) \log p_\theta(x) + 
  n\hspace{-.4em} \sum_{x, x'\in\mathcal{X}} q(x' \cond x) Q_n(x) 
  \Bigl[\log p_\theta(x \cond x') - \log q(x' \cond x) \Bigr] . 
 \end{equation}
 Expanding terms we obtain
 \begin{align}
  & L_B(\theta, q, n) = 
  \sum_{x\in\mathcal{X}} \bigl[q_{n+1}(x) + n Q_n(x)\bigr] \log p_\theta(x) - \nonumber \\
  & n \sum_{x,x'\in\mathcal{X}} q(x' \cond x) Q_n(x) \log p_\theta(x') - 
  n \sum_{x\in\mathcal{X}} Q_n(x) \kldiv*{q(x' \cond x)}{p^r_\theta(x' \cond x)} .
 \end{align}
 Recalling the definition
 \begin{equation}
  q_k(x) = \sum_{x'\in\mathcal{X}} q(x \cond x') q_{k-1}(x'), \ \ \ 
  Q_n(x) = \frac{1}{n} \sum_{k=1}^n q_k(x) ,
 \end{equation}
 it is easy to prove that the first two terms collapse to the data likelihood
 \begin{equation}
 L(\theta) = \sum_{x\in\mathcal{X}} q_1(x) \log p_\theta(x) .
\end{equation}
\end{proof}

\PZero*
\begin{proof}
 If the Markov chain with kernel $q(x \cond x')$ has the limiting distribution $p_{\theta}(x)$, then by construction $q_n(x)$ converges to this distribution. It follows that the first term in \eqref{eq:chain} converges to 
 \begin{equation}
    \sum_{x\in\mathcal{X}} q_{n+1}(x) \log p_\theta(x) 
    \xrightarrow{n\rightarrow\infty}
    \sum_{x\in\mathcal{X}} p_{\theta}(x) \log p_\theta(x),
 \end{equation}
 Its gradient for $q_{\infty}(x) = p_{\theta}(x)$ (at a fixed $\theta$) expresses as 
 \begin{align}
    &\nabla_\theta \sum_{x\in\mathcal{X}} q_{\infty}(x) \log p_\theta(x)
    = \sum_{x\in\mathcal{X}} p_{\theta}(x) \nabla_\theta \log p_\theta(x) \notag\\
    & = \sum_{x\in\mathcal{X}} p_{\theta}(x) \frac{1}{p_{\theta}(x)} \nabla_\theta p_\theta(x)
    = \sum_{x\in\mathcal{X}} \nabla_\theta p_\theta(x)\notag = \nabla_\theta \sum_{x\in\mathcal{X}} p_\theta(x)
    = \nabla_\theta 1 = 0.
 \end{align}
 Moreover, we can bound the norm of the gradient for finite $n$ by
 \begin{align}
  \norm*{\nabla_\theta \sum_{x\in\mathcal{X}} q_n(x) \log p_{\theta}(x)}_1 
  & = \norm*{\sum_{x\in\mathcal{X}} [q_n(x) - p_{{\theta}}(x)] 
  \nabla_{\theta} \log p_{\theta}(x)}_1 \notag\\
  & \leqslant \norm*{q_n - p_{{\theta}}}_{TV} \max_{x\in\mathcal{X}} 
  \norm*{\nabla_\theta\log p_{\theta}(x)}_1 .
 \end{align}
 The claim is then a direct consequence of the convergence of $q_n$ to its limiting distribution, being equal $p_{{\theta}}$. This convergence is geometric: $\norm*{q_n - p_{{\theta}}}_{TV} \leqslant \widetilde{C} \alpha^n$, where $\alpha\in(0,1)$ depends on the mixing time of the kernel $q(x'\cond x)$.
 The constant $C$ is then equal to $\widetilde{C} \max_{x\in\mathcal{X}} \norm*{\nabla_\theta\log p_{\theta}(x)}_1$ for given $\theta$.
\end{proof}
\subsubsection{Further used facts.}\label{sec:alg0}
The expression for the gradient estimate~\eqref{eq:stoch-est} is obtained as follows. Let us expand the gradient of the second term of~\eqref{eq:chain}, \ie of
\begin{align}
\hat F = n\sum_{x, x'\in\mathcal{X}} q(x' \cond x) Q_n(x) \Bigl[\log p_\theta(x \cond x') - \log q(x' \cond x) \Bigr].
\end{align}
Since $q$ is not varied with $\theta$, the gradient of $\log q(x' \cond x)$ is zero. Substituting $Q_n(x) = \frac{1}{n}\sum_{k=1}^n q_k(x)$ we obtain
\begin{align}
\nabla_\theta \hat F & = \sum_{x, x'\in\mathcal{X}} q(x'\cond x) \sum_{k=1}^n q_k(x)  \nabla_\theta \log p_\theta(x \cond x') \notag\\
& = \sum_{k=1}^n \sum_{x, x'\in\mathcal{X}} q(x'\cond x) q_k(x)  \nabla_\theta \log p_\theta(x \cond x').
\end{align}
If we sample $x^1 \sim q_1(x)$, $x^{k+1} \sim q(x' \cond x^{k})$, then $(x^{k+1}, x^{k})$ is a sample from $q(x'\cond x) q_k(x)$. Therefore, $g$ in~\eqref{eq:stoch-est} is an unbiased stochastic gradient estimate of $\hat F$.

The fact that the reverse chain has the same limiting distribution as the forward one is verified as follows. We have by definition
$$
p^r_\theta(x' \cond x) = p_\theta(x \cond x')p_\theta(x')/p_\theta(x)
$$
and
$$
p_\theta(x) = \sum_{x'}p_\theta(x')p_\theta(x \cond x').
$$
We substitute and verify that 
\begin{eqnarray}
& & \sum_{x}p_\theta(x) p^r_\theta(x' \cond x) = \nonumber \\
& & \sum_{x}p_\theta(x) p_\theta(x \cond x')p_\theta(x')/p_\theta(x)
= \sum_{x} p_\theta(x \cond x')p_\theta(x') = p_\theta(x').
\end{eqnarray}

\subsection{Proof of the Theorem}\label{app:proof-T1}
First, we recall our setup. We consider a collection of conditional distributions $p_{\theta_i}(x_i \cond x_{-i})$, $i=1,\ldots,m$, where $x_i$ is a particular variable from $x = (x_1, x_2\ldots x_m)$, and $x_{-i}$ summarises all variables except $x_i$. The Markov kernel of the corresponding MCMC process is defined as
\begin{equation}\label{eq:app_gibbs_transition}
    p_\theta(x \cond x') =
    \sum_{i=1}^m \alpha_i \:
    p_{\theta_i}(x_i \cond x_{-i})
    \iverson{x_{-i}=x'_{-i}} ,
\end{equation}
and its limiting distribution is denoted by $p_\theta(x)$. The stationarity condition for this kernel implies that its limiting distribution has the form of a mixture
\begin{equation}\label{eq:mixture}
 p_\theta(x) = \sum_{i=1}^m 
 \alpha_i p_{\theta_i}(x_i \cond x_{-i}) p_\theta(x_{-i}) .
\end{equation}
It is not necessarily the case that {\em limiting conditionals} $p_\theta(x_i \cond x_{-i})$ coincide with the {\em model conditionals} $p_{\theta_i}(x_i \cond x_{-i})$.

We call a collection of conditional distributions {\em consistent} if there exists a joint distribution having them as conditionals.

\Tconsistency*

\begin{proof}
 We begin the proof with the last statement. If the conditionals of the stationary distribution coincide with the model conditionals, then the latter are consistent with the joint being $p_\theta(x)$.

 If the model conditionals $p_{\theta_i}(x_i \cond x_{-i})$ are consistent, \ie have a common joint $\bar p(x)$, then the MCMC with the kernel \eqref{eq:app_gibbs_transition} is a standard Gibbs sampler for $\bar p(x)$, its limiting distribution is $\bar p(x) = p_\theta (x)$, and its limiting conditionals are $\bar p(x_i \cond x_{-i}) = p_{\theta}(x_i \cond x_{-i}) = p_{\theta_i}(x_i \cond x_{-i})$. In this case, the mixture \eqref{eq:mixture} does not depend on the weights $\alpha_i$ because all its components coincide. The standard Gibbs sampler is also known to satisfy the detailed balance.
 
 It remains to prove the reverse, \ie assuming detailed balance for the MCMC kernel \eqref{eq:app_gibbs_transition}, we prove that the collection of conditionals is consistent.
 Let us consider a particular $x_i$, its complement $x_{-i}$, and the corresponding conditional $p_{\theta_i}(x_i \cond x_{-i})$. We aim at proving $p_{\theta_i}(x_i \cond x_{-i}) = p_\theta(x_i \cond x_{-i})$. The detailed balance constraint expressed in these components reads
    \begin{equation}\label{eq:app_dbc}
        p_\theta(x_i, x_{-i})\cdot p_\theta(x'_i,x'_{-i} \cond x_i, x_{-i}) = 
        p_\theta(x'_i, x'_{-i})\cdot p_\theta(x_i,x_{-i} \cond x'_i, x'_{-i}) .
    \end{equation}
 Recall that our kernel \eqref{eq:app_gibbs_transition} has the property that in each transition either $x_{-i}$ or $x_i$ can change but not both simultaneously. 

 In the former case, $x_{-i} \neq x'_{-i}$, the $i$-th addend in \eqref{eq:app_gibbs_transition} is zero for any $(x, x')$ due to the vanishing indicator $\iverson{x_{-i}=x'_{-i}}$. So, we can not draw any conclusion about $p_{\theta_i}$ from \eqref{eq:app_dbc}.

 Let us consider the latter case $x_{-i} = x'_{-i}$. If $x_i = x'_i$, the left-hand side and the right-hand side of \eqref{eq:app_dbc} are equal for any $(x,x')$. Hence, again we can not draw any conclusion about $p_{\theta_i}$ from \eqref{eq:app_dbc}.

 It remains only to consider the case $x_{-i} = x'_{-i}$ and $x_i \neq x'_i$. We substitute \eqref{eq:app_gibbs_transition} into \eqref{eq:app_dbc}. All addends in the sum of \eqref{eq:app_gibbs_transition} except the $i$-th one vanish (as the corresponding indicators are zero). The factor $\alpha_i$ on both sides cancels. Therefore, \eqref{eq:app_dbc} is reduced to
        \begin{equation}
            p_\theta(x_i, x_{-i})\cdot p_{\theta_i}(x'_i \cond x_{-i}) = 
            p_\theta(x'_i, x_{-i})\cdot p_{\theta_i}(x_i \cond x_{-i}) ,
        \end{equation}
        or, equivalently
        \begin{equation}
            p_\theta(x_i\cond x_{-i})\cdot p_{\theta_i}(x'_i \cond x_{-i}) = 
            p_\theta(x'_i\cond x_{-i})\cdot p_{\theta_i}(x_i \cond x_{-i}) .
        \end{equation}
        It follows
        \begin{equation}
            \frac{p_{\theta_i}(x_i\cond x_{-i})}{p_{\theta_i}(x'_i\cond x_{-i})} = 
            \frac{p_\theta(x_i\cond x_{-i})}{p_\theta(x'_i\cond x_{-i})}\ \ \forall x_i,x'_i,x_{-i},
        \end{equation}
        which finally means $p_{\theta_i}(x_i\cond x_{-i}) = p_\theta(x_i\cond x_{-i})$. As $p_\theta(x_i\cond x_{-i})$ are consistent per definition with the joint being $p_\theta(x)$, so are also $p_{\theta_i}(x_i\cond x_{-i})$.

    \end{proof}

\subsection{Examples}\label{app:examples}
We mentioned in \cref{rem:diffusion} that denoising diffusion is a prominent example showing that a linearly growing gap between the likelihood and the lower bound in \eqref{eq:chain_kl} does not necessarily lead to a breakdown of the approach. 
\begin{example}\label{ex:diffusion}
 Denoising diffusion \citep{Ho:NeurIPS2020,Rombach:CVPR2022} fixes $q(x \cond x')$ to a simple ``noisifying'' model, whose limiting distribution is independent standard Gaussian noise, \ie $\lim_{n\to\infty} q_n = \mathcal{N}(0,\mathbb{I})$.\footnote{We consider a slightly simplified version of the diffusion model, assuming that both, the forward and reverse processes are homogeneous Markov chains.} This has the advantage, that the distributions $q_n(x)$ can be computed in closed form. The disadvantage of the diffusion approach is that the lower bound is no longer tight and the corresponding gap increases linearly with the sampling length. This follows from \eqref{eq:chain_kl} because the KL-divergence $\kldiv*{q(x' \cond x)}{p^r_\theta(x' \cond x)}$ is strictly positive since the limiting distributions of $p_\theta(x\cond x')$ and $q(x \cond x')$ are different by definition.
 Despite the convincing experimental results, it remains unclear whether the method maximises the true objective function, \ie the data likelihood, when maximising the lower bound with the chosen fixed diffusion process.
\end{example}

The next examples show that the proposed \cref{alg:multi} is similar to known techniques in simple cases.

\begin{example} \label{ex:PLLH} Consider the supervised case, \ie with complete realisations $x = (x_1, x_2\ldots x_m)$. If we stop the reverse chain at $n=1$, the optimised objective becomes
\begin{equation}
 L_P(\theta) = \frac{1}{m} \sum_{i=1}^m \log p_{\theta_i}(x_i \cond x_{-i}) + \text{\it const},
\end{equation}
which corresponds to {\em pseudo-likelihood} \citep{Besag:1975}. Notice that in the non-parametric limit, this enforces consistency simply because the empirical data distribution has by itself consistent conditionals.
\end{example}

\begin{example}\label{ex:EF}
 Let us consider the task of learning an EF-Harmonium \citep{Welling:NIPS2004}
 \begin{equation}\label{eq:ef_harmonium}
    p_W(x,y) = \frac{1}{Z(W)}\exp \scalp*{\psi(x)}{W\xi(y)} ,
\end{equation}
with sufficient statistics $\psi$ and $\xi$, in the fully supervised case. A stochastic gradient of $\log p_W(x \cond y)$ can be obtained with
\begin{align}\label{eq:ef_stgrad}
    & \hat{x} \sim p_W(x \cond y), \nonumber \\
    & \hat \nabla_W \log p_W(x \cond y) = \psi(x)\otimes \xi(y) - \psi(\hat x)\otimes \xi(y).
\end{align}
The gradient of $\log p_W(y \cond x)$ is similar. This model is consistent by design and the sampler \eqref{eq:gibbs_transition} satisfies detailed balance.
Applying the proposed learning algorithm results in the {\em contrastive divergence} algorithm \citep{Hinton:ICANN1999}. This is because almost all gradient terms for the sequence $((x_1,y_1),\ldots,(x_n,y_n))$ sampled from \eqref{eq:gibbs_transition} cancel out, which leaves the gradient estimated for the chain as
\begin{equation}\label{eq:cd}
    g = \psi(x_1)\otimes\xi(y_1) - \psi(\hat x_n)\otimes\xi(\hat y_n).
\end{equation}
\end{example}

\section{Validation Experiments: Dependency on the Chain Length} \label{app:artificial}

We present additional details and further experiments accomplishing results from \cref{subsec:artificial}

\begin{figure}[ht]
    \begin{center}
        \includegraphics[width=0.9\linewidth]{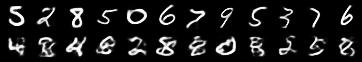}
    \end{center}
    \caption{\label{fig:app_art1} Given a code $z$, images are obtained by  sampling, which alternates between sampling from $p(x\cond c,z)$ and $p(c\cond x,z)$, \ie $x\sim p(x\cond z)$. Top row: $z\sim\pi(z)$ (see the text), bottom row: $z$ uniform.}
\end{figure}

We used the marginal distribution $p(z)$ from our MNIST experiments as the target probability distribution $\pi(z)$. Recall that in the MNIST experiments we learned three conditional probability distributions $p(x\cond c,z)$, $p(c\cond x,z)$ and $p(z\cond x,c)$, where $x$ are binarised MNIST images, $c$ are classes (digits), and $z$ are binary latent variables. The trained model admits sampling from $\pi(z)$ as follows: (i) pick an example $(x,c)$ from the MNIST training set, and (ii) sample $z\sim p(z\cond x,c)$. \cref{fig:app_art1} illustrates that the obtained $\pi(z)$ is not trivial.

In \cref{subsec:artificial} we presented a dependency of the model consistency (addressed implicitly through the matrix symmetry) on the chain length for a simple log-linear model \eqref{eq:art_mat}. Here we present a more detailed view, i.e.~we want to inspect the dependency of the optimisation objective on $n$ during training. However, when defining the generic algorithm, we have omitted an intractable term $\sum_x q_{n+1}(x) \log p_\theta(x)$, which has an asymptotically vanishing gradient but a non-negligible value. In order to approximate the value of $L_B(\theta, q , n)$ more accurately we use the objective
\begin{eqnarray}\label{hatF}
& & F(\theta, q, n) = \sum_{x,x'\in\mathcal{X}} q(x' \cond x) q_n(x) \log q(x' \cond x) + \nonumber \\
 & & \ \ n \hspace{-.4em} \sum_{x, x'\in\mathcal{X}} q(x' \cond x) Q_n(x) \Bigl[\log p_\theta(x \cond x') - \log q(x' \cond x) \Bigr] ,
\end{eqnarray}
see \cref{sec:LB-value} for the rationale. 

We perform three experiments:
\begin{enumerate}
    \item Using log-linear transition probability \eqref{eq:art_mat}.
    \item Recall that \eqref{eq:art_mat} satisfies detailed balance iff the matrix $W$ is symmetric. In the second experiment we enforce detailed balance by setting $W = \frac{1}{2}(W + W^T)$ after each update iteration.
    \item Instead of the log-linear form \eqref{eq:art_mat} we use an MLP for $p(z_t\cond z_{t-1})$ having one hidden layer with $128$ units and $\tanh$-activations.
\end{enumerate}
For each of these three cases we perform the learning two times. First with short chains, using $n=1$, and second with longer chains, using $n=10$.

\begin{figure*}[p!]
    \begin{center}
        \includegraphics[width=0.43\textwidth]{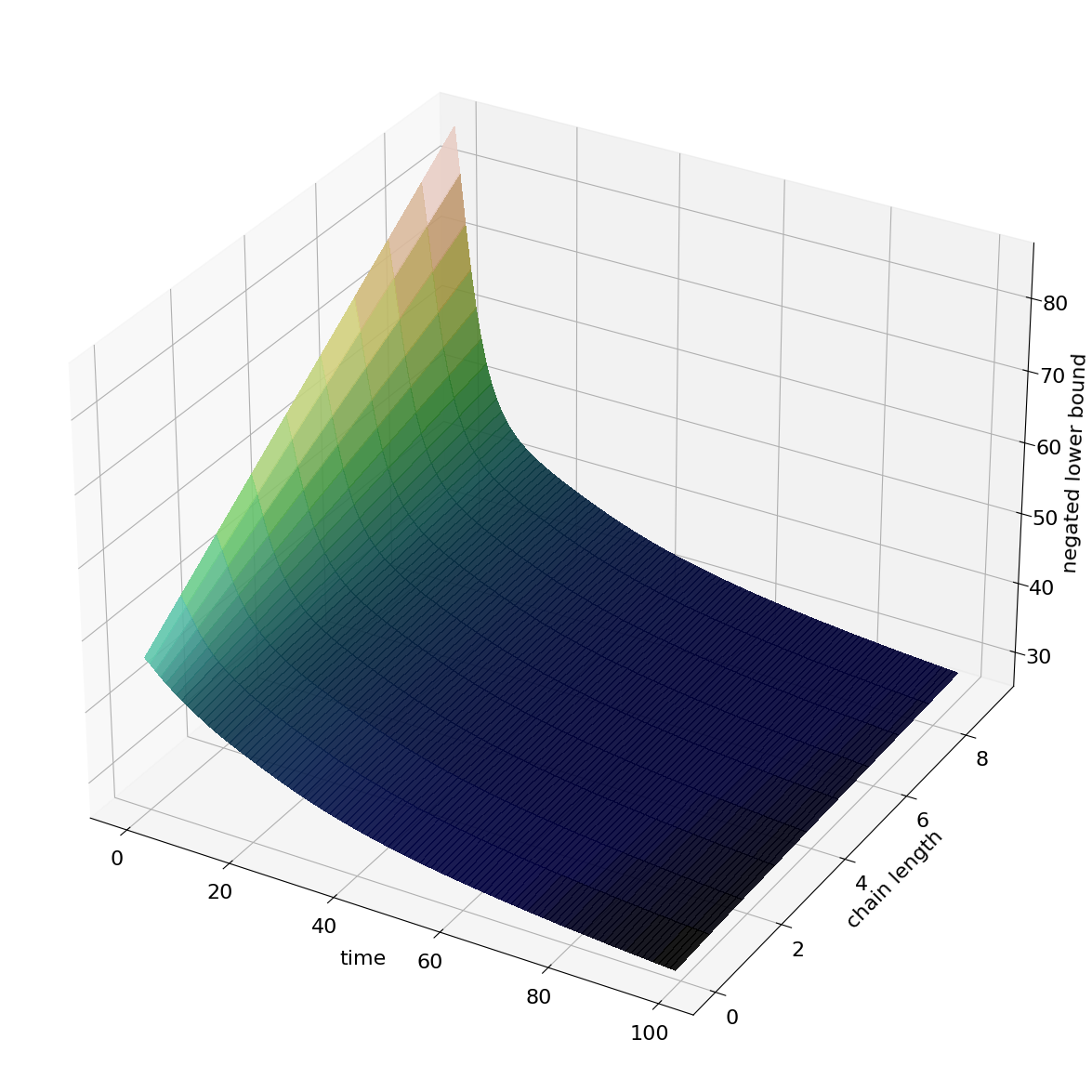}\hspace{10pt}
        \includegraphics[width=0.43\textwidth]{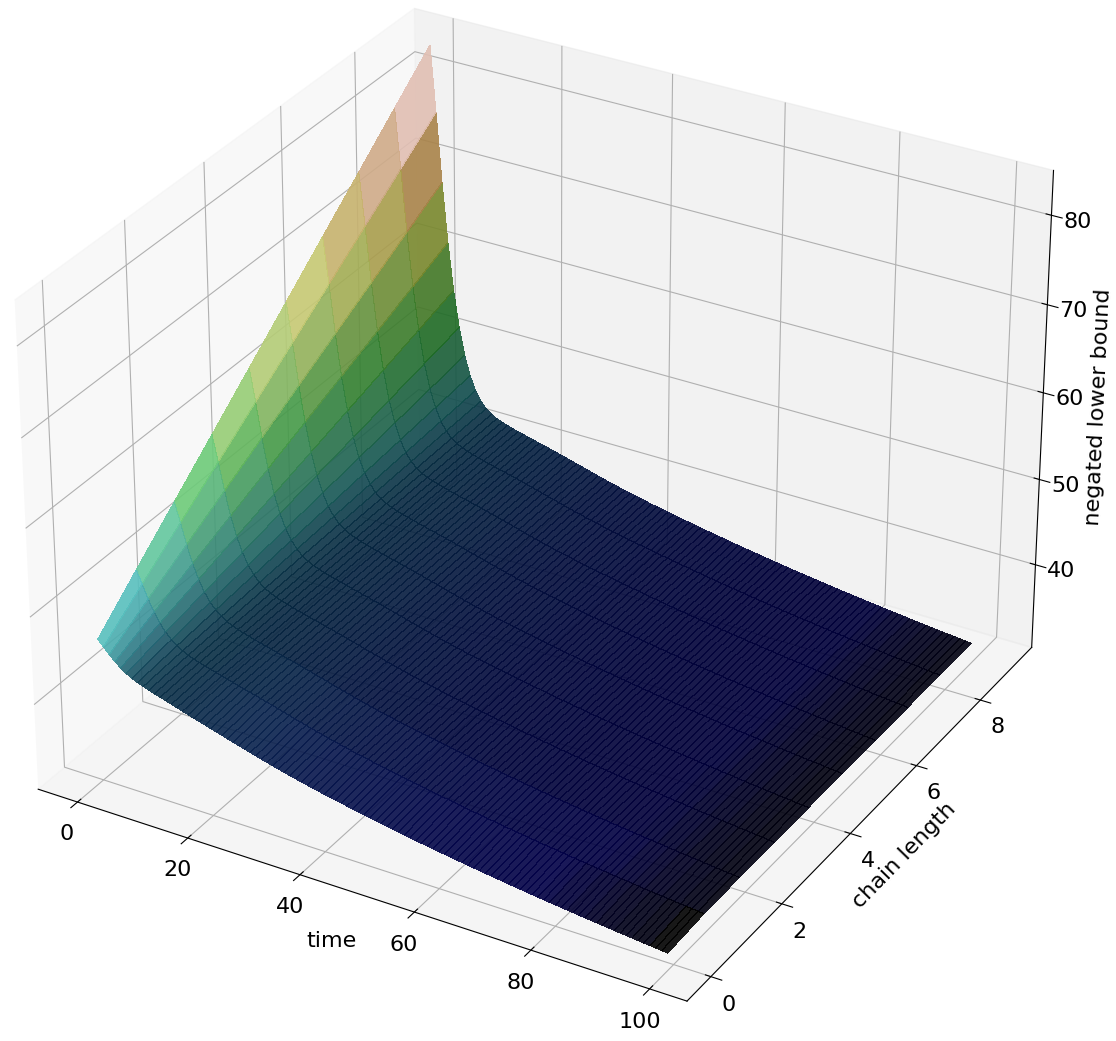}

        \vspace{-10pt}
        \includegraphics[width=0.43\textwidth]{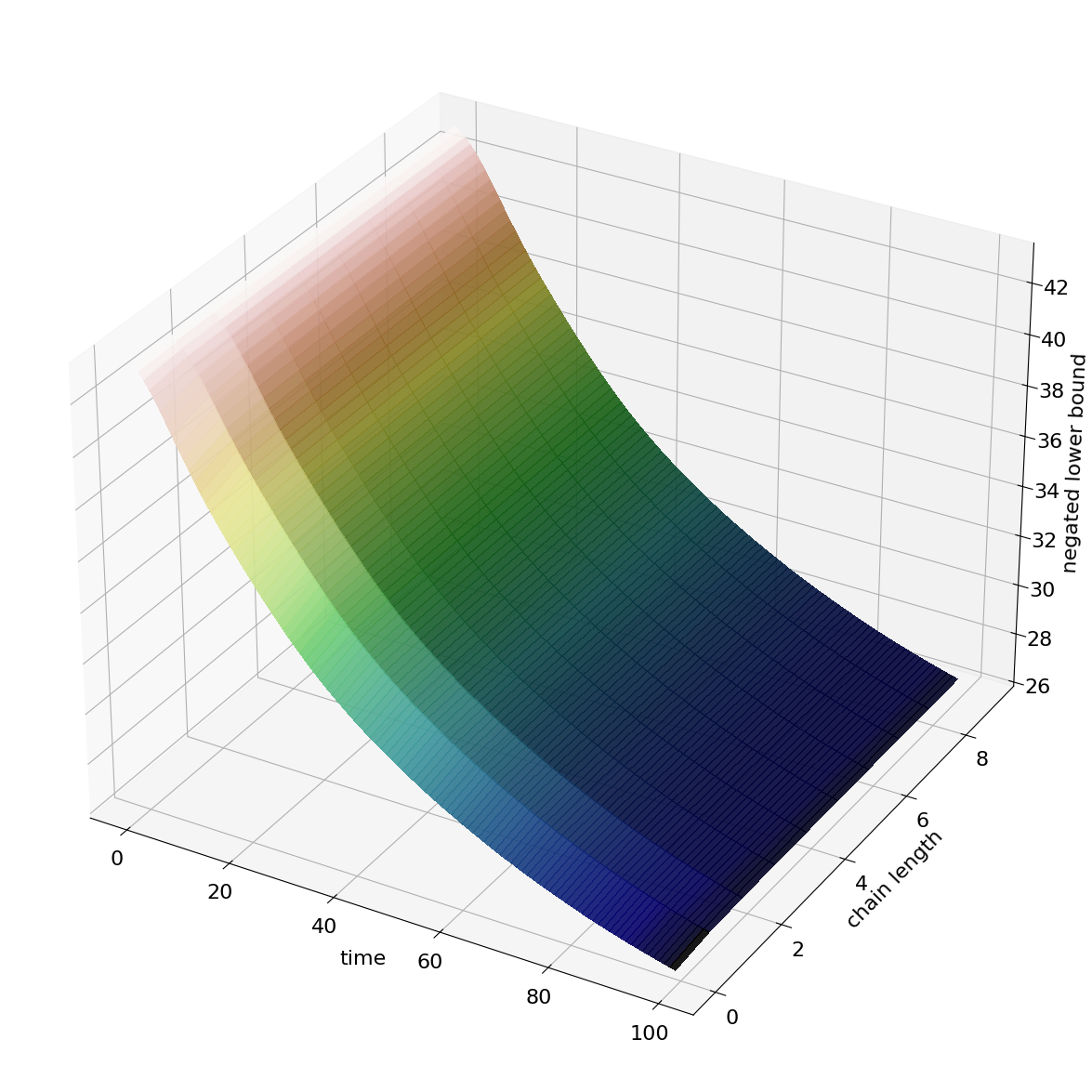}\hspace{10pt}
        \includegraphics[width=0.43\textwidth]{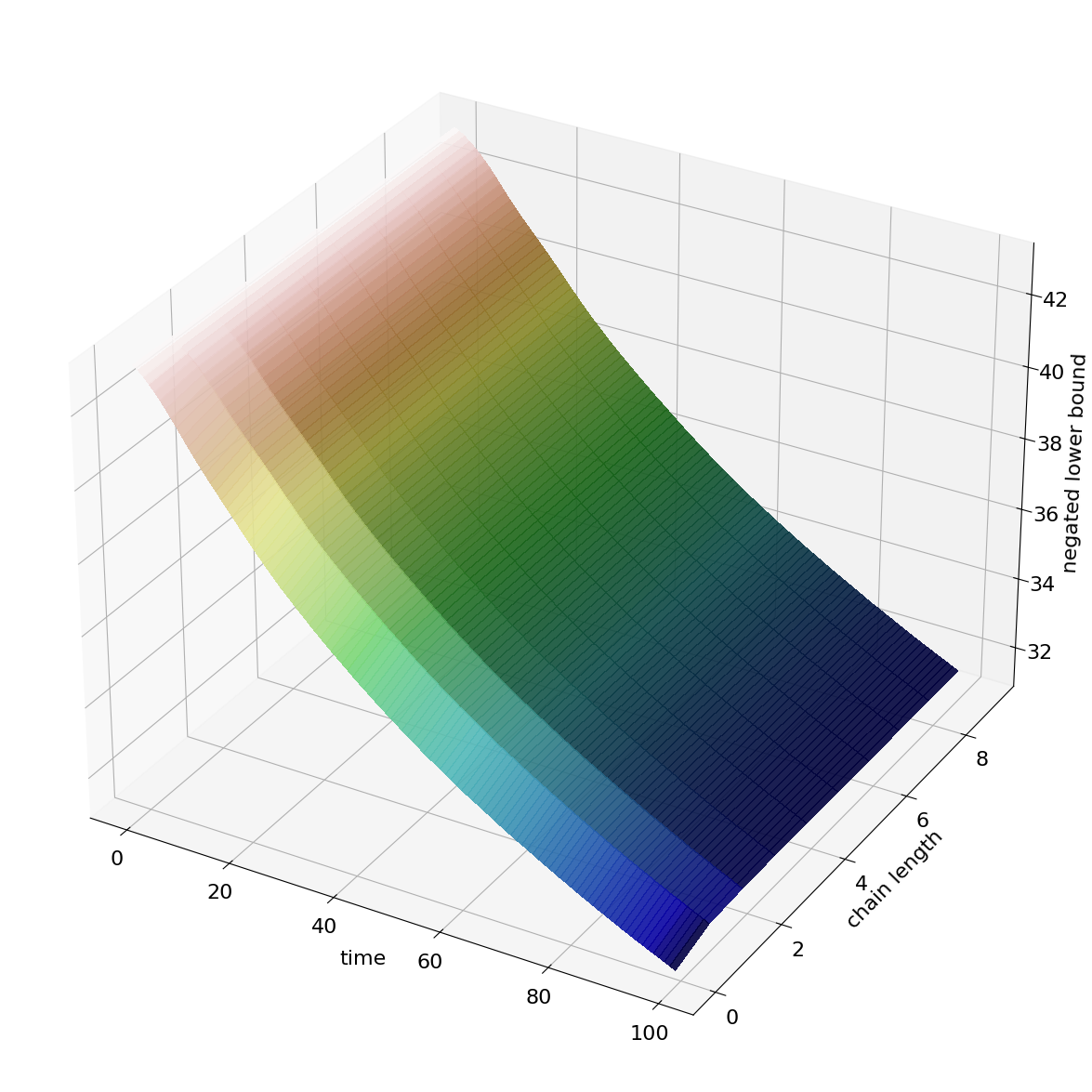}

        \vspace{-10pt}
        \includegraphics[width=0.43\textwidth]{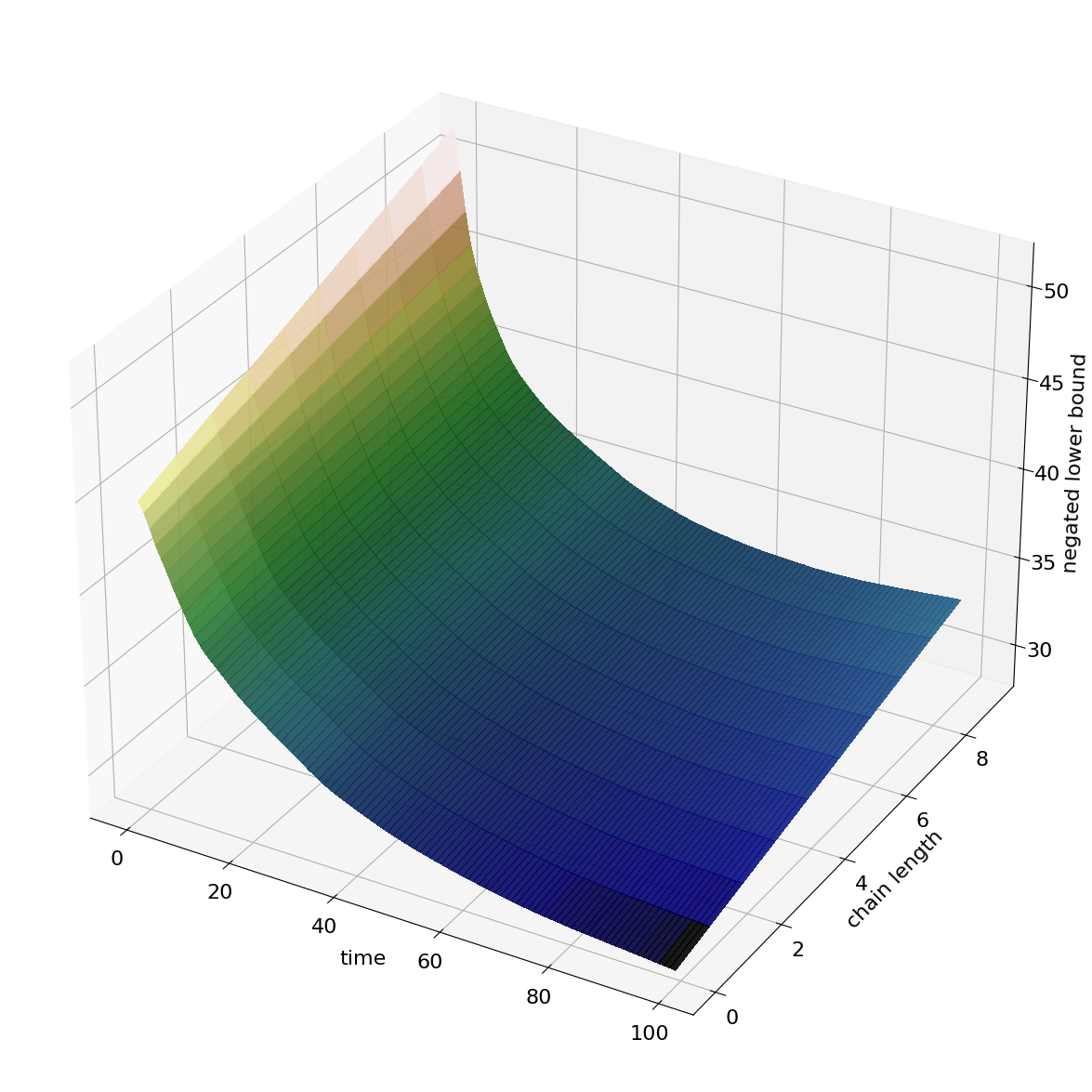}\hspace{10pt}
        \includegraphics[width=0.43\textwidth]{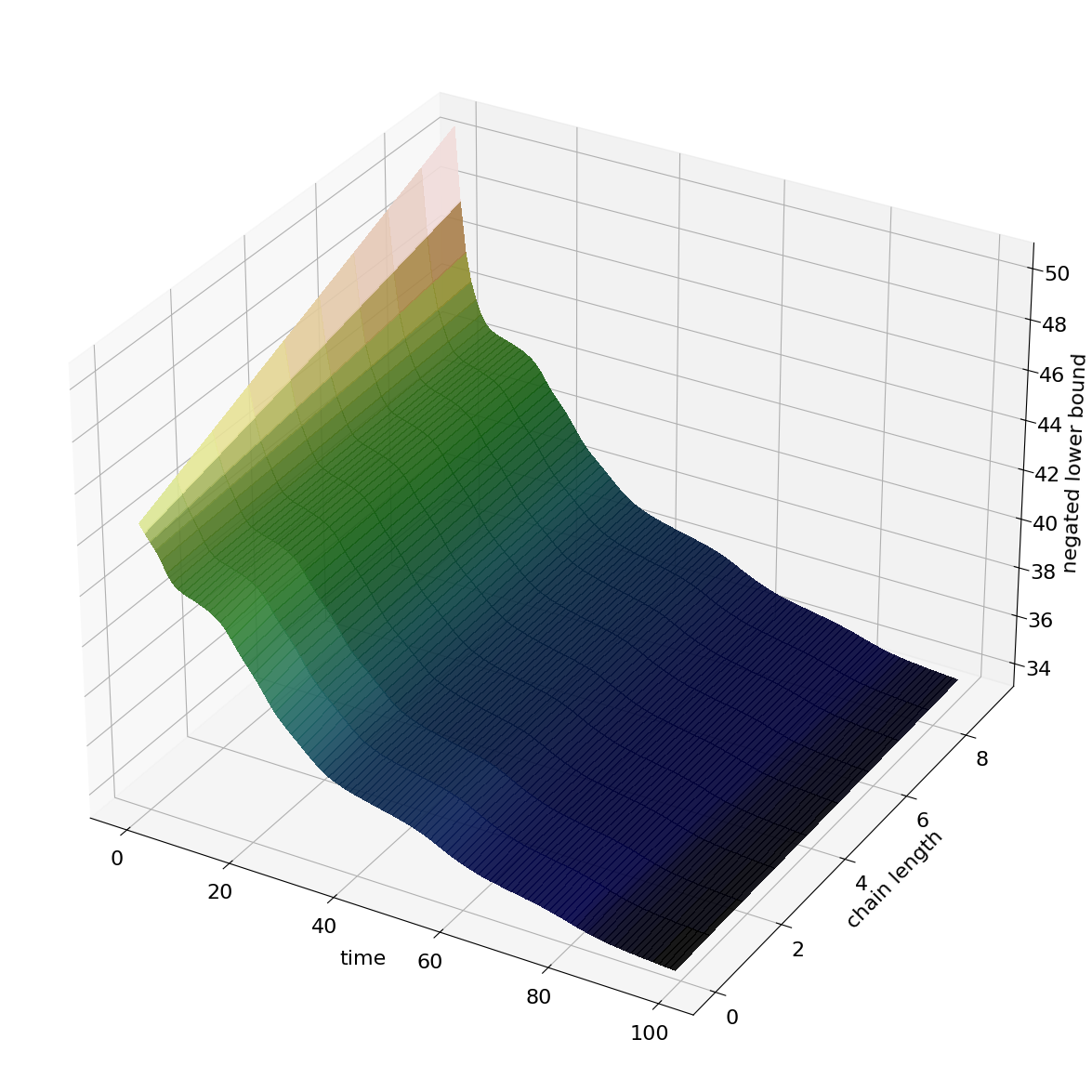}
    \end{center}
    \caption{\label{fig:app_art2} The dependencies of the optimisation objective on the chain length during the training. Top row: log-linear model, second row: symmetrised log-linear model, bottom row: MLP. Left: short optimisation chains, right: long optimisation chains.}
\end{figure*}

The results are presented in \cref{fig:app_art2}. For both optimizing with $n=1$ and optimizing with $n=10$, we plot the values of the objective \eqref{hatF} for all $n$-s in dependency on the training time. The results confirm our expectations. When detailed balance is not explicitly enforced (experiments (1) and (3)), the objective grows linearly with $n$ at the beginning of the learning, \ie when the parameters are randomly initialised, and hence, the details balance is not satisfied. As conjectured, our algorithm implicitly enforces detailed balance, \ie the graphs are almost horizontal along the $n$-direction at the end of the learning process. Sometimes, if detailed balance is ``easy to achieve'' (the first experiment), short chains enforce detailed balance almost as good as long ones. In general, however, enforcing detailed balance requires longer chains (see experiment (3)). Finally, if detailed balance is satisfied by design (experiment (2)), the graphs are horizontal along the $n$-direction from the very beginning. In this case, short and long chains work equally good.

\subsection{Approximating the value of LB}\label{sec:LB-value}
A natural choice of approximating $L_B(\theta, q, n)$ would be just to drop the intractable  term $\sum_x q_{n+1}(x) \log_{\theta}(x)$, \ie, to approximate $L_B$ as
\begin{align}\label{eq:subject}
   & \hat F(\theta, q, n) =  n \sum_{x, x'\in\mathcal{X}} q(x' \cond x) Q_n(x) 
   \log \left[p_\theta(x \cond x') - q(x' \cond x)\right].
\end{align}
However, even when the gradient of this omitted term asymptotically vanishes for $n \rightarrow \infty$, its value does not.
It is asymptotically the value of the entropy of $p_\theta(x)$. To approximate it a bit better we add back a tractable averaged conditional entropy of $q$, which results in the approximation 
\begin{eqnarray}\label{eq:subject1}
    & & F(\theta, q, n) =  \sum_{x, x'\in\mathcal{X}} q(x' \cond x) q_n(x) 
    \log q(x' \cond x) + \nonumber \\
    & & n \sum_{x, x'\in\mathcal{X}} q(x' \cond x) Q_n(x) 
    \log \left[p_\theta(x \cond x') - q(x' \cond x)\right]
 \end{eqnarray}
(see \eqref{hatF}). We can formally show the following.
\begin{proposition}\label{Prop:hatF}
For discrete variables, asymptotically for $n\rightarrow \infty$, $F$~\eqref{eq:subject1} is a better approximation of the lower bound \eqref{eq:chain} than $\hat F$~\eqref{eq:subject}.
\end{proposition}
\begin{proof}
Consider $q(x'\cond x) = p_{\theta}(x'\cond x)$ and $\theta$ fixed. In the limit $n\rightarrow \infty$, $q_n(x)$ approaches $p_{\theta}(x)$ (see proof of \cref{P:zero} above). Then, for $\hat F$, the approximation can be expressed as
\begin{align}
    \hat F(n) - L_B(n) \rightarrow -\sum_{x} q_n(x) \log q_n(x),
\end{align}
which is the entropy $H$ of $q_n(x)$. For discrete variables, $H$ is non-negative and therefore $\hat F$ is asymptotically an overestimate of $L_B$.

For $F$, we obtain
\begin{align}
    F(n) - L_B(n) \rightarrow \sum_{x,x'} q_n(x, x') \log \frac{q_n(x, x')}{q_n(x)q_n(x')},
\end{align}
which is the mutual information $I$ between $x$ and $x'$ drawn from $q_n(x, x') = q_n(x' \cond x) q_n(x)$. Since $I$ is non-negative, this shows that $F$ also asymptotically overestimates $L_B$.

There holds $H \geq I$ because their difference $H - I$ is the averaged conditional entropy, which is non-negative for discrete variables. It follows that asymptotically 
\begin{align}
L_B \leq F \leq \hat F
\end{align}
and hence $F$ is more accurate.

We also obtained the following interpretation of the approximation $F$: it is (asymptotically) accurate when the chain has weak coupling (mutual information between consecutive states is small), \ie when it mixes fast.
\end{proof}

\section{Network Architectures} \label{app:nets}

Here we describe network architectures used in our experiments. We focus on those used for CelebA in \cref{subsec:celeba}. Models for MNIST and Fashion MNIST consist of very similar building blocks.

\begin{figure}[h]
    \begin{center}
        \includegraphics[width=0.4\textwidth]{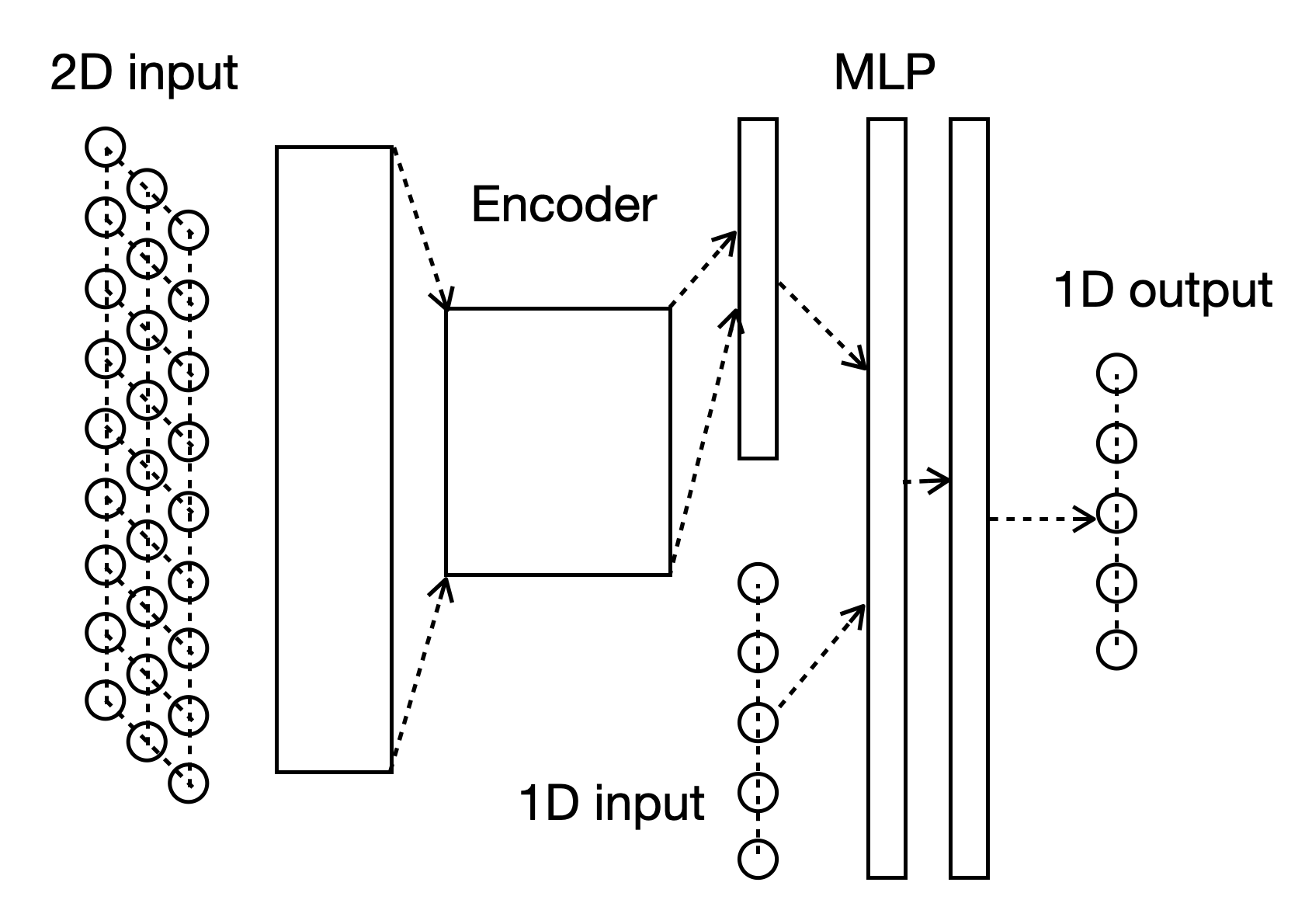}\hspace{2ex}
        \includegraphics[width=0.47\textwidth]{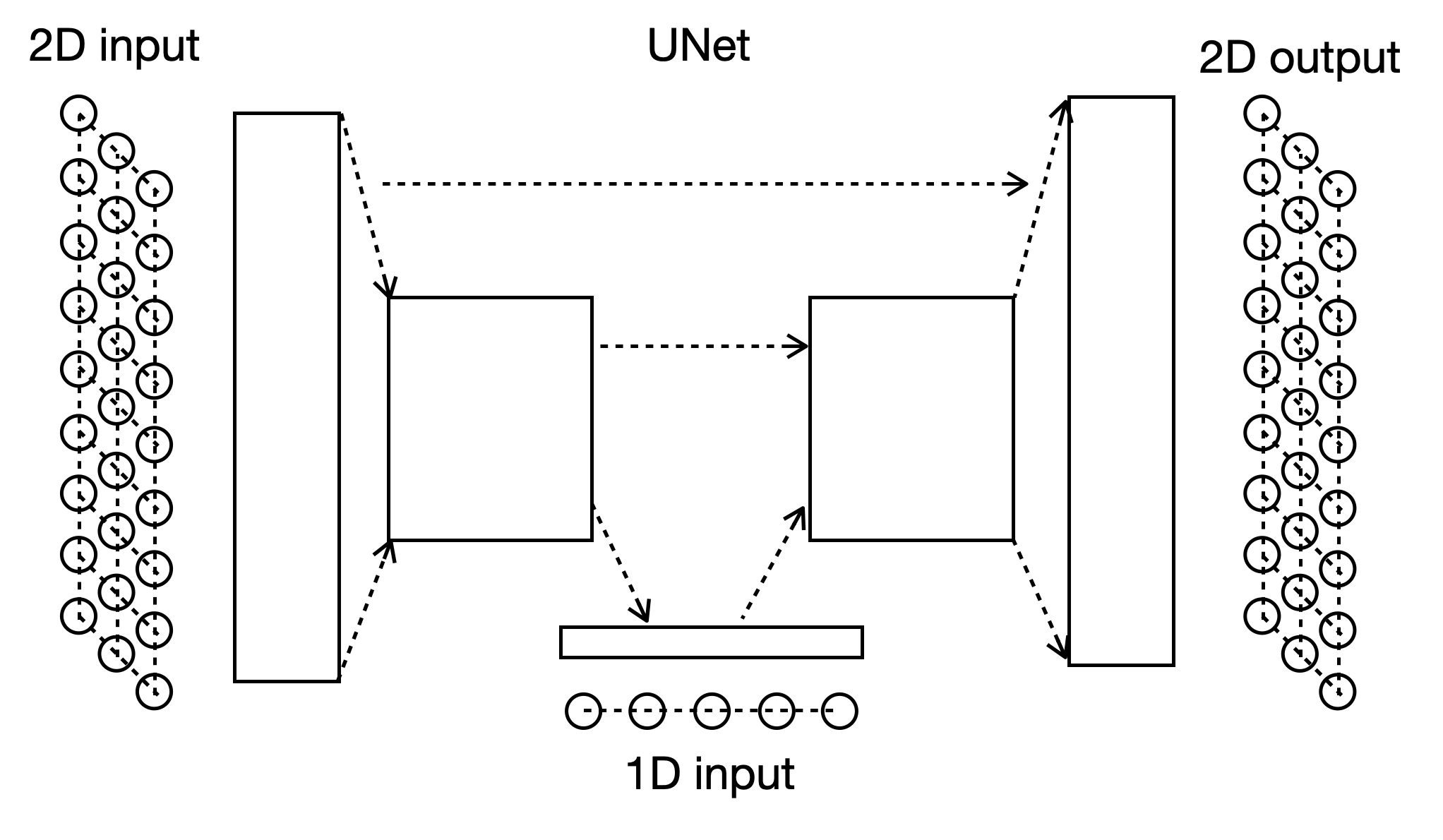}
    \end{center}
    \caption{\label{fig:nets} Network architectures. Left: for 1D outputs, right: for 2D outputs.}
\end{figure}

In CelebA, we have two types of networks. The first one, shown in \cref{fig:nets} (left), gets both 1D and 2D inputs and produces a 1D output, for example $p(y\cond c, x, z, s)$. Thereby, several 1D inputs are concatenated, 2D inputs are concatenated along the channel direction. First, the 2D input goes through an encoder-like network, \ie a convolutional network of decreasing spatial resolution. The encoder consists of 9 convolutional layers. Some of them are performed with strides, effectively reducing spatial resolution. The output of the encoder is a feature vector, which is then concatenated with the 1D input. The resulting vector goes through an MLP consisting of 3 fully connected layers. The output values are interpreted according to the used output type, \ie for which component they are used. For attributes $c$ and binary latent variables $y$, the network outputs the corresponding logits. For continuous output $z$, it provides components means and logarithms of the variances.

The second network type, shown in \cref{fig:nets} (right), gets again both 1D and 2D inputs, but produces a 2D output, \eg $p(s\cond c, x, y, z)$. It has a UNet-like architecture, \ie it consists of an encoder and a decoder, equipped with additional skip-connections. The encoder is of the same architecture as for the first network type, the decoder has the inverse architecture, \ie the same number of layers, hidden units, etc., using, however, transposed convolutions. The 1D input is concatenated to the UNet's hidden layer of the lowest, \ie $1{\times}1$, resolution.

We used the Adam optimiser with a gradient step size $1e{-}4$ in all our experiments. Further details can be found in the source code, which is available at {\tt https://github.com/dschles70/paracond}.

\end{document}